\newif\iffull
\fulltrue

\documentclass[11pt,twoside]{article}
\usepackage{xspace,fullpage,amsfonts,boxedminipage,microtype,hyperref}
\usepackage{paralist}
\usepackage{bm}
\usepackage{bbm}
\usepackage{array}
\usepackage{multirow}
\usepackage{times}

\usepackage{color}
\usepackage[style=alphabetic,backend=bibtex,maxbibnames=20,maxcitenames=6,firstinits=true,doi=false,url=false]{biblatex}
\newcommand*{\citet}[1]{\AtNextCite{\AtEachCitekey{\defcounter{maxnames}{2}}}\textcite{#1}}

\newcommand*{\citep}[1]{\cite{#1}}

\bibliography{vf-allrefs-local,holdout}

\usepackage{microtype}
\usepackage{graphicx}
\usepackage{amsmath,amsthm,amssymb,bbm,color}
\usepackage{subfigure}
\usepackage{booktabs}
\usepackage{algorithm,algorithmic}
\usepackage{vfmacros}

\newif\ifnotes
\notesfalse
\ifnotes
        \newcommand{\marnote}[1]{{\sf \textcolor{blue}{#1}}}
\else
		\newcommand{\marnote}[1]{}
\fi

\newcommand{\vnote}[1]{\marnote{Vitaly: #1}}

\newcommand{\mote}[1]{\marnote{Moritz: #1}}
\newcommand{\mnote}{\mote}

\providecommand{\acc}{\mathsf{acc}}
\providecommand{\ind}{\mathsf{Ind}}
\providecommand{\opt}{\mathsf{Opt}}

\providecommand{\cO}{{\mathcal O}}

\providecommand{\by}{{\bar y}}
\providecommand{\bz}{{\bar z}}
\providecommand{\bq}{{\bar q}}
\providecommand{\br}{{\bar r}}
\providecommand{\bp}{{\bar p}}
\providecommand{\bs}{{\bar s}}
\providecommand{\bc}{{\bar c}}

\providecommand{\bpi}{{\bar \pi}}
\providecommand{\balpha}{{\bar \alpha}}
\newtheorem{fact}[thm]{Fact}

\newcommand{\conf}{\mathsf{conf}}
\newcommand{\pl}{\mathsf{plu}}
\newcommand{\Bin}{\mathsf{Bin}}
\newcommand{\mnom}{\mathsf{Mnom}}
\newcommand{\NB}{\mathtt{NB}}
\newcommand{\Pois}{\mathsf{Pois}}

\usepackage{hyperref}

\begin{document}

\title{The advantages of multiple classes for reducing overfitting from test set reuse}
\author{
Vitaly Feldman\thanks{Google Brain. Part of this work was done while the author was visiting the Simons Institute for the Theory of Computing.}
\and Roy Frostig\thanks{Google Brain}
\and Moritz Hardt\thanks{University of California, Berkeley. Work done while at Google.}
}
\maketitle

\begin{abstract}
Excessive reuse of holdout data can lead to overfitting. However, there is little concrete evidence of significant overfitting due to holdout reuse in popular multiclass benchmarks today.
Known results show that, in the worst-case, revealing the accuracy of $k$ adaptively chosen classifiers on a data set of size $n$ allows to create a classifier with bias of $\Theta(\sqrt{k/n})$ for any binary prediction problem. We show a new upper bound of $\tilde O(\max\{\sqrt{k\log(n)/(mn)},k/n\})$ on the worst-case bias that any attack can achieve in a prediction problem with $m$ classes. Moreover, we present an efficient attack that achieve a bias of $\Omega(\sqrt{k/(m^2 n)})$ and improves on previous work for the binary setting ($m=2$). We also present an inefficient attack that achieves a bias of $\tilde\Omega(k/n)$. Complementing our theoretical work, we give new practical attacks to stress-test multiclass benchmarks by aiming to create as large a bias as possible with a given number of queries. Our experiments show that the additional uncertainty of prediction with a large number of classes indeed mitigates the effect of our best attacks.

Our work extends developments in understanding overfitting due to adaptive data analysis to multiclass prediction problems. It also bears out the surprising fact that multiclass prediction problems are significantly more robust to overfitting when reusing a test (or holdout) dataset. This offers an explanation as to why popular multiclass prediction benchmarks, such as ImageNet, may enjoy a longer lifespan than what intuition from literature on binary classification suggests.
\end{abstract}

\section{Introduction}

Several machine learning benchmarks have shown surprising longevity, such as the ILSVRC 2012 image classification benchmark based on the ImageNet database~\cite{russakovsky2015imagenet}. Even though the test set contains only $50{,}000$ data points, hundreds of results have been reported on this test set. Large-scale hyperparameter tuning and experimental trials across numerous studies likely add thousands of queries to the test data. Despite this excessive data reuse, recent replication studies~\cite{RechtRSS18, RechtRSS19} have shown that the best performing models transfer rather gracefully to a newly collected test set collected from the same source according to the same protocol.

What matters is not only the number of times that a test (or holdout) set has been accessed, but also how it is accessed. Modern machine learning practice is \emph{adaptive} in its nature. Prior information about a model's performance on the test set inevitably influences future modeling choices and hyperparameter settings. Adaptive behavior, in principle, can have a radical effect on generalization.

Standard concentration bounds teach us to expect a maximum error of $O(\sqrt{\log(k)/n})$ when estimating the means of $k$ non-adaptively chosen bounded functions on a data set of size $n.$ However, this upper bound sharply deteriorates to $O(\sqrt{k/n})$ for adaptively chosen functions, an exponential loss in~$k$. Moreover, there exists a sequence of adaptively chosen functions, what we will call an \emph{attack}, that causes an estimation error of $\Omega(\sqrt{k/n})$ ~\cite{DworkFHPRR14:arxiv}.

What this means is that in principle an analyst can overfit substantially to a test set with relatively few queries to the test set. Powerful results in \emph{adaptive data analysis} provide sophisticated holdout mechanisms that guarantee better error bounds through noise addition~\cite{DworkFHPRR15:science} and limited feedback mechanisms~\cite{BlumH15}.
However, the standard holdout method remains widely used in practice, ranging from machine learning benchmarks and data science competitions to validating scientific research and testing products during development.
If the pessimistic bound were indicative of performance in practice, the holdout method would likely be much less useful than it is.

It seems evident that there are factors that prevent this worst-case overfitting from happening in practice. In this work, we isolate the number of classes in the prediction problem as one such factor that has an important effect on the amount of overfitting we expect to see.
Indeed, we find that in the worst-case the number of queries required to achieve certain bias grows at least linearly with the number of classes, a phenomenon that we establish theoretically and substantiate experimentally.

\subsection{Our contributions}

We study in both theory and experiment the effect that multiple classes have on the amount of overfitting caused by test set reuse. In doing so, we extend important developments for binary prediction to the case of multiclass prediction.

To state our results more formally, we introduce some notation. A classifier is a mapping $f\colon X\to Y,$ where $Y=[m]=\{1,\dots,m\}$ is a discrete set consisting of $m$ classes and $X$ is the data domain. A data set of size~$n$ is a tuple $S \in (X\times Y)^n$ consisting of $n$ labeled examples $(x_i, y_i)_{i\in[n]}$, where we assume each point is drawn independently from a fixed underlying population.
In our model, we assume that a data analyst can \emph{query} the data set by specifying a classifier~$f\colon X\to Y$ and observing its accuracy $\acc_S(f)$ on the data set $S$, which is simply the fraction of points that are correctly labeled $f(x_i)=y_i.$ We denote by $\acc(f)=\Pr\{f(x)=y\}$ the accuracy of $f$ over the underlying population from which $(x,y)$ are drawn. Proceeding in $k$ rounds, the analyst is allowed to specify a function in each round and observe its accuracy on the data set. The function chosen at a round~$t$ may depend on all previously revealed information. The analyst builds up a sequence of adaptively chosen functions~$f_1,\dots,f_k$ in this manner.

We are interested in the largest value that $\acc_S(f_t)-\acc(f_t)$ can attain over all $1\le t \le k$. Our theoretical analysis focuses on the worst case setting where an analyst has no prior knowledge (or, equivalently, has a uniform prior) over the correct label of each point in the test set. In this setting,
the highest expected accuracy achievable on the unknown distribution is $1/m$. In effect, we analyze the expected advantage of the analyst over random guesses.
 
In reality, an analyst typically has substantial prior knowledge about the labels and starts out with a far stronger classifier than one that predicts at random. Using domain information, models, and training data, there are many conceivable ways to label many points with high accuracy and to pare down the set of labels for points the remaining points. Indeed, our experiments explore a couple of techniques for reducing label uncertainty given a good baseline classifier. After incorporating all prior information, there is usually still a large set of points for which there remains high uncertainty over the correct label. Effectively, to translate the theoretical bounds to a practical context, it is useful to think of the dataset size $n$ as the number of point that are hard to classify, and to think of the class count $m$ as a number of (roughly equally likely) candidate labels for those points.

Our theoretical contributions are several upper and lower bounds on the achievable bias in terms of the number of queries~$k$, the number of data points~$n$, and the number of classes~$m$. We first establish an upper bounds on the bias achievable by any attack in the uniform prior setting. 

\begin{thm}[Informal]
\label{thm:upper-bound-few-queries-intro}
There is a distribution $P$ over examples labeled by $m$ classes such that any algorithm that makes at most $k$ queries to a dataset $S\sim P^n$ must satisfy with high probability
\[
\max_{1\le t\le k}\acc_S(f_t) = \fr{m} + O \lp\max\left\{ \sqrt{\frac{k \log n}{nm}},  \frac{k \log n}{n}\right\} \rp\,.
\]
\end{thm}
This bound has two regimes that emerge from the concentration properties of the binomial distribution. The more important regime for our discussion is when $k=\tilde O(n/m)$ for which the bound is $\tilde O(\sqrt{k/(nm)})$. In other words, achieving the same bias requires $O(m)$ more queries than in the binary case. What is perhaps surprising in this bound is that the difficulty of overfitting is not simply due to an increase in the amount of information per label. The label set~$\{1, \dots, m\}$ can be indexed with only $\log(m)$ bits of information.

We remark that these bounds hold even if the algorithm has access to the data points without the corresponding labels. The proofs follow from information-theoretic compression arguments and can be easily extended to any algorithm for which one can bound the amount of information extracted by the queries (e.g.~via the approach in \citep{DworkFHPRR15:arxiv}).

Complementing this upper bound, we describe two attack algorithms that establish lower bounds on the bias in the two parameter regimes.

\begin{thm}[Point-wise attack, informal]
\label{thm:nb-attack-intro}
For sufficiently large $n$ and $ n\geq k \geq k_{\min} = O(m \log m)$ there is an attack that uses $k$ queries and on any dataset $S$ outputs $f$ such that
\[
\acc_S(f) = \fr{m} + \Omega\lp \sqrt{\frac{k}{nm^2}}\rp \,.
\]
\end{thm}

The algorithm underlying Theorem~\ref{thm:nb-attack-intro} outputs a classifier that computes a weighted plurality of the labels that comprise its queries, with weights determined by the per-query accuracies observed.
Such an attack is rather natural, in that the function it produces is close to those produced by boosting and other common techniques for model aggregation.
It also allows for simple incorporation of any prior distribution over a label of each point.
In addition, it is adaptive in the relatively weak sense: all queries are independent from one another except for the final classifier that combines them.

This attack is computationally efficient and we prove that it is optimal within a broad class of attacks that we call \emph{point-wise}. Roughly speaking, such an attack predicts a label independently for each data point rather than reasoning jointly over the labels of multiple points in the test set.
The proof of Theorem \ref{thm:nb-attack-intro} requires a rather delicate analysis of the underlying random process.

Theorems~\ref{thm:upper-bound-few-queries-intro} and~\ref{thm:nb-attack-intro} leave open a gap between bounds in the dependence on $m$.
We conjecture that our analysis of the attack in Theorem~\ref{thm:nb-attack-intro} is asymptotically optimal and thus, considering the optimality of the attack, gives a lower bound for all point-wise attacks.
If correct, this conjecture suggests that the effect of a large number of labels on mitigating overfitting is even more pronounced for such attacks. Some support for this conjecture is given in our experimental section (Figure~\ref{fig:synth-bias}).

Our second attack is based on an algorithm that exactly reconstructs the labels on a subset of the test set.
\begin{thm}[Reconstruction-based attack, informal]
\label{thm:indentify-intro}
For any $k= \Omega(m \log m)$, there exists an attack $\A$ with access to test set points such that $\A$ uses $k$ queries and on any dataset $S$ outputs $f$ such that $$\acc_S(f) = \min\left\{1, \fr{m} + \Omega\lp \frac{ k \log(k/m)}{n \log m} \rp \right\} .$$
\end{thm}
The attack underlying Theorem~\ref{thm:indentify-intro} requires knowledge of the test points (but not their labels)---in contrast to a point-wise attack like the previous---and is not computationally efficient in general. For some $t\leq n$ it reconstructs the labeling of the first $t$ points in the test set using queries that are random over the first $t$ points and fixed elsewhere.
The value $t$ is chosen to be sufficiently small so that the answers to $k$ random queries are sufficient to uniquely identify, with high probability, the correct labeling of $t$ points jointly.
This analysis builds on and generalizes the classical results of \citet{ErdosR63} and \citet{Chvatal83}.
A natural question for future work is whether a similar bias can be achieved without identifying test set points and in polynomial time (currently a polynomial time algorithm is only known for the binary case \citep{Bshouty09}).

\paragraph{Experimental evaluation.}
The goal of our experimental evaluation is to come up with effective attacks to stress-test multiclass benchmarks. We explore attacks based on our point-wise algorithm in particular. Although designed for worst-case label uncertainty, the point-wise attack proves applicable in a realistic setting once we reduce the set of points and the set of labels to which we apply it.

What drives performance in our experiments is the kind of prior information the attacker has. In our theory, we generally assumed a \emph{prior-free} attacker that has no a priori information about the labels in the test set. In practice, an analyst almost always knows a model that performs better than random guessing. We therefore split our experiments into two parts: (i) simulations in the prior-free case, and (ii) effective heuristics for the ImageNet benchmark when prior information is available in the form of a well-performing model.

Our prior-free simulations it becomes substantially more difficult to overfit as the number of classes grows, as predicted by our theory.
Under the same simulation, restricted to two classes, we also see that our attack improves on the one proposed in~\citep{BlumH15} for binary classification.

Turning to real data and models, we consider the well-known 2012 ILSVRC benchmark based on ImageNet~\cite{russakovsky2015imagenet}, for which the test set consist of $50{,}000$ data points with $1000$ labels. Standard models achieve accuracy of around $75$\% on the test set. It makes sense to assume that an attacker has access to such a model and will use the information provided by the model to overfit more effectively. We ignore the trained model parameters and only use the model's so-called \emph{logits}, i.e., the predictive scores assigned to each class for each image in the test set. In other words, the relevant information provided by the the model is a $50{,}000\times1000$ array.

But how exactly can we use a well-performing model to overfit with fewer queries? We experiment with three increasingly effective strategies:
\begin{enumerate}
    \item The attacker uses the model's logits as the prior information about the labels. This gives only a minor improvement over a prior-free attack.
    \item The attacker uses the model's logits to restrict the attack to a subset of the test set corresponding to the lowest ``confidence'' points. This strategy gives modest improvements over a prior-free attack.
    \item The attacker can exploit the fact that the model has good top-$R$ accuracy, meaning that, for every image, the $R$ highest weighted categories are likely to contain the correct class label. The attacker then focuses only on selecting from the top $R$ predicted classes for each point. For $R=2$, this effectively reduces class count to the binary case.
\end{enumerate}

In absolute terms, our best performing attack overfits by about $3$\% with $5000$ queries.

Naturally, the multiclass setting admits attacks more effective than the prior-free baseline. However, even after making use of the prior, the remaining uncertainly over multiple classes makes overfitting harder than in the binary case. Such attacks also require more sophistication and hence it is natural to suspect that they are less likely to be the accidental work of a well-intentioned practitioner.

\subsection{Related work}
The problem of biasing results due to adaptive reuse of the test data is now well-recognized. Most relevant to us are the developments starting with the work of~\citet{DworkFHPRR14:arxiv, DworkFHPRR15:science} on reusable holdout mechanisms.
In this work, noise addition and the tools of differential privacy serve to improve the $\sqrt{k/n}$ worst-case bias of the standard holdout method to roughly $k^{1/4}/\sqrt{n}$.
The latter requires a strengthened generalization bound due to~\cite{BassilyNSSSU16}.
Separately, computational hardness results suggest that no trivial accuracy is possible in the adaptive setting for $k>n^2$~\cite{HardtU14,SteinkeU15}.

Blum and Hardt~\cite{BlumH15} developed a limited feedback holdout mechanism, called the Ladder algorithm, that only provides feedback when an analyst improves on the previous best result significantly.
This simple mechanism leads to a bound of $\log(k)^{2/3}/n^{1/3}$ on what they call the \emph{leaderboard error}.
With the help of noise addition, the bound can be improved to $\log(k)^{3/5}/n^{2/5}$~\cite{Hardt17arxiv}. Blum and Hardt also give an attack on the standard holdout mechanism that achieves the $\sqrt{k/n}$ bound for a binary prediction problem.

Accuracy on a test set is an average of accuracies at individual points.
Therefore our attacks on the test set are related to the vast literature on (approximate) recovery from linear measurements, which we cannot adequately survey here (see for example \citep{vershynin2015estimation}).
The primary difference between our work and the existing literature is the focus on the multiclass setting, which no longer has the simple linear structure of the binary case.
(In the binary case the accuracy measurement is essentially an inner product between the query and the labels viewed in $\{\pm 1\}$.)
In addition, even in the binary case the closest literature (see below) focuses the analysis on prediction with high accuracy (or small error) whereas we focus on the regime where the advantage over random guessing is relatively small.

Perhaps the closest in spirit to our work are database reconstruction attacks in the privacy literature.
In this context, it was first demonstrated by~\citet{DinurN03} that sufficiently accurate answers to $O(n)$ random linear queries allow exact reconstruction of a binary database with high probability. Many additional attacks have been developed in this context allowing more general notions of errors in the answers (e.g.~\citep{dwork2007price}) and specific classes of queries (e.g.~\citep{kasiviswanathan2010price,kasiviswanathan2013power}).
To the best of our knowledge, this literature does not consider queries corresponding to prediction accuracy in the multiclass setting and also focuses on (partial) reconstruction as opposed to prediction bias. Defenses against reconstruction attacks have lead to the landmark development of the notion of differential privacy \citep{DworkMNS:06}.

Another closely related problem is reconstruction of a pattern in $[m]^n$ from accuracy measurements.
For a query $q \in [m]^n$, such a measurement returns the number of positions in which $q$ is equal to the unknown pattern. In the binary case ($m=2$), this problem was introduced by~\citet{Shapiro60} and was studied in combinatorics and several other communities under a variety of names, such as ``group testing'' and ``the coin weighing problem on the spring scale'' (see \citep{Bshouty09} for a literature overview).
In the general case, this problem is closely related to a generalization of the Mastermind board game~\citep{Wiki:mastermind} with only black answer pegs used.
\citet{ErdosR63} demonstrated that the optimal reconstruction strategy in the binary case uses $\Theta(n/\log n)$ measurements. An efficient algorithm achieving this bound was given by~\citet{Bshouty09}. General $m$ was first studied by~\citet{Chvatal83} who showed a bound of $O(n\log m/\log(n/m))$ for $m \leq n$ (see~\citet{doerr2016playing} for a recent literature overview). It is not hard to see that the setting of this reconstruction problem is very similar to our problem when the attack algorithm has access to the test set points (and only their labels are unknown).
Indeed, the analysis of our reconstruction-based attack (Theorem~\ref{thm:indentify-intro}) can be seen as a generalization of the argument from~\citet{ErdosR63,Chvatal83} to partial reconstruction.
In contrast, our point-wise attack does not require such knowledge of the test points and it gives bounds on the achievable bias (which has not been studied in the context of pattern reconstruction).

An attack on a test set is related to a boosting algorithm. The goal of a boosting algorithm is to output a high-accuracy predictor by combining the information from multiple low-accuracy ones. A query function to the test set that has some correlation with the target function gives a low-accuracy predictor on the test set and an attack algorithm needs to combine the information from these queries to get the largest possible prediction accuracy on the test set. Indeed, our optimal point-wise attack (Theorem~\ref{thm:nb-attack-intro}) effectively uses the same combination rule as the seminal Adaboost algorithm \citep{FreundSchapire:97} and its multiclass generalization \citep{hastie2009multi}. Note that in our setting one cannot modify the weights of individual points in the test set (as is required by boosting). On the other hand, unlike a boosting algorithm, an attack algorithm can select which predictors to use as queries. Another important difference is that boosting algorithms are traditionally analyzed in the setting when the algorithm achieves high-accuracy, whereas we deal primarily with the more delicate low-accuracy regime.

\section{Preliminaries}
Let $S=(x_i,y_i)_{i\in [n]}$ denote the test set, where $(x_i,y_i) \in X \times Y$. Let $m =|Y|$ and without loss of generality we assume that $Y = [m]$. For $f \colon X \to Y$ its accuracy on the test set is $\acc_S(f) = \fr{n} \sum_{i\in [n]} \ind(f(x_i) = y_i)$
We are interested in overfitting attack algorithms that do not have access to the test set $S$. Instead, they have query access to accuracy on the test set $S$, i.e.\ for any classifier $f \colon X \to Y$ the algorithm can obtain the value $\acc_S(f)$. We refer to each such access as a query, and we denote the execution of an algorithm $\A$ with access to accuracy on the test $S$ and $\A^{\cO(S)}$. In addition, in some settings the attack algorithm may also have access to the set of points $x_1,\ldots,x_n$.

A $k$-query test set overfitting attack is an algorithm that, given access to at most $k$ accuracy queries on some unknown test set $S$, outputs a function $f$.  For any such possibly randomized algorithm $\A$ we define
$$\acc(\A) \doteq \inf_{S \in (X\times Y)^n} \E_{f=\A^{\cO(S)}}[\acc_S(f)] .$$
An algorithm is non-adaptive if none of its queries depend on the accuracy values of previous queries (however the output function depends on the accuracies so a query for that function is adaptive).

The main attack we design will be from a restricted class of {\em point-wise} attacks. We define an attack is {\em point-wise} if its queries and output function are generated for each point individually (while still having access to accuracy on the entire dataset). More formally, $\A$ is defined using an algorithms $\B$ that evaluated queries and the final classifier. A query $f_\ell$ at $x$ is defined as the execution of $\B$ on values $f_1(x),\ldots,f_{\ell-1}(x)$ and the corresponding accuracies: $\acc_S(f_1),\ldots,\acc_S(f_{\ell-1})$. Similarly, for $k$ query attack, the value of the final classifier $f$ at $x$ is defined as the execution of $\B$ on $f_1(x),\ldots,f_{k}(x)$ and $\acc_S(f_1),\ldots,\acc_S(f_{k})$. An important property of point-wise attacks is that they can be easily implemented without access to data points. Further, the accuracy they achieve depends only on the vector of target labels.

Our upper bounds on the bias will apply even to algorithms that have access to points $x_1,\ldots,x_n$. The accuracy of such algorithms  depends only on target labels. Hence for most of the discussion we describe the test set by the vector of labels $\by = (y_1,\ldots,y_n)$. Similarly, we specify each query by a vector of labels on the points in the dataset $\bq = (q_1,\ldots,q_n) \in [m]^n$. Accordingly, we use $\by$ in place of the test set and $\bq$ in place of a classifier in our definitions of accuracy and access to the oracle (e.g.~$\acc_\by(\bq)$ and $\A^{\cO(\by)}$).

In addition to worst-case (expected) accuracy, we will also consider the average-case accuracy of the attack algorithm on randomly sampled labels. The random choice of labels may reflect the uncertainty that the attack algorithm has about the labels. Hence it is natural to refer to it as a prior distribution. In general, the prior needs to be specified on all points in $X$, but for point-wise attacks or attacks that have access to points it is sufficient to specify a vector $\bpi = (\pi_1,\ldots,\pi_n)$, where each $\pi_i$ is a probability mass function on $[m]$ corresponding to the prior on $y_i$. We use $\by \sim \bpi$ to refer to $\by$ being chosen randomly with each $y_i$ sampled independently from $\pi_i$.
We let $\mu_m^n$ denote the uniform distribution over $[m]^n$.
We also define the average case accuracy of $\A$ relative to $\bpi$ by
$$\acc(\A,\bpi) \doteq \E_{\by \sim \bpi}\lb \E_{\br=\A^{\cO(\by)}}[\acc_\by(\br)]\rb .$$
Note that for every $\bpi$, $\acc(\A) \leq \acc(\A,\bpi)$.

For a matrix of query values $Q \in [m]^{n\times k}$, $i\in [n]$ and $j \in [k]$, we denote by $Q^j$ the $j$-th column of the matrix (which corresponds to query $j$) and by $Q_i$ the $i$-th row of the matrix: $(Q_{i,1},\ldots,Q_{i,k})$ (which corresponds to all query values for point $i$). For a matrix of queries $Q$ and label vector $\by$ we denote by $\acc_\by(Q)\doteq (\acc_\by(Q_j))_{j\in [k]}$.

\iffull
\subsection{Random variables and concentration}

For completeness we include several standard concentration inequalities that we use below.
\begin{lem}[(Multiplicative) Chernoff bound]
\label{lem:chernoff}
Let $X$ be the average of $n$ i.i.d. Bernoulli random variables with bias $p$.
Then for $\alpha \in (0,1)$ $$\pr[X  \geq (1+\alpha)p] \leq e^{\frac{-\alpha^2 pn}{2+\alpha}} \mbox{ and }$$  $$\pr[X  \leq (1-\alpha)p] \leq e^{-\frac{\alpha^2 pn}{2}}.$$
\end{lem}
We also state the Berry-Esseen theorem for the case of Bernoulli random variables.
\begin{lem}
\label{lem:berry-esseen}
Let $X$ be the average of $n$ i.i.d. Bernoulli random variables with bias $p\leq 1/2$. Then for every real $v$,
$$\left| \pr[X  \leq v] - \pr[\zeta \leq v] \right| = O\lp\frac{1}{\sqrt{pn}}\rp ,$$ where $\zeta$ is distributed according to the Gaussian distribution with mean $p$ and variance $p(1-p)$.
\end{lem}
\fi

\section{Upper bound}
In this section we formally establish the upper bound on bias that can be achieved by any overfitting attack on a multiclass problem. The upper bound assumes that the attacker does not have any prior knowledge about the test set. That is, its prior distribution is uniform over all possible labelings. 

The upper bound applies to algorithms that have access to the points in the test set. The upper bound has two distinct regimes. For $k = \tilde O(n/m)$ the upper bound on bias is $O \lp \sqrt{\frac{k \log n}{nm}} \rp$ and so the highest bias achieved in this regime is $\tilde{O}(1/m)$ (i.e.\ total accuracy improves by at most a constant factor). For $k\geq n/m$, the upper bound is $O \lp \frac{k \log n}{n} \rp$. Note that, in this regime, the attacker pays on average one query to improve the accuracy by one data point (up to log factors).

The proof of the upper bound relies on a simple description length argument, showing that finding a classifier with desired accuracy and non-negligible probability of success requires learning many bits about the target labeling.
\begin{thm}
\label{thm:upper-bound}
Let $m,n,k$ be positive integers and $\mu_m^n$ denote the uniform distribution over $[m]^n$. Then for every $k$-query attack algorithm $\A$, $\delta > 0$, $b=k\ln(n+1) +\ln(1/\delta)$, and
$$\eps = 2 \cdot \max \left\{\sqrt{\frac{b}{nm}},\frac{b}{n} \right\},$$
$$\pr_{\by \sim \mu_m^n, \br = \A^{\cO(\by)}}  \lb \acc_\by(\br) \geq \fr{m} + \eps \rb  \leq \delta.$$
\end{thm}
\begin{proof}
  We first observe that for any fixed labeling $\br$, $\acc_\by(\br)$ for $\by$ chosen randomly according to $\mu_m^n$ is distributed as the average of $n$ independent Bernoulli random variables with bias $1/m$. By the Chernoff bound, for any fixed labeling $\br$,
  $$\pr_{\by \sim \mu_m^n}\lb \acc_\by(\br) \geq \fr{m} + \eps \rb \leq e^{-\frac{mn\eps^2}{2+m \eps}} .$$ Therefore for any fixed distribution $\rho$ over $[m]^n$, we have
  \equ{\pr_{\br \sim \rho, \by \sim \mu_m^n}\lb \acc_\by(\br) \geq \fr{m} + \eps \rb \leq e^{-\frac{mn\eps^2}{2+m \eps}} . \label{eq:bias}}
  Consider the execution of $\A$ with responses of the accuracy oracle fixed to some sequence of values $\alpha=(\alpha_1,\ldots,\alpha_k)\in \{0,1/n,\ldots,1\}^k$. We denote the resulting algorithm by $\A^\alpha$. It output distribution is fixed (that is independent of $\by$). Therefore by eq.~\eqref{eq:bias} we have:
  \equn{\pr_{\br = \A^\alpha, \by \sim \mu_m^n}\lb \acc_\by(\br) \geq \fr{m} + \eps \rb \leq e^{-\frac{mn\eps^2}{2+m \eps}} . \label{eq:bias-a}}
  We denote the set $\{0,1/n,\ldots,1\}^k$ of possible values of $\alpha$ by $V$. Note that $|V| \leq (n+1)^k$ and thus we get:
  \equn{\sum_{\alpha \in V} \pr_{\br = \A^\alpha, \by \sim \mu_m^n}\lb \acc_\by(\br) \geq \fr{m} + \eps \rb \leq (n+1)^k \cdot e^{-\frac{mn\eps^2}{2+m \eps}} .}
  Clearly, for every $\by$, the accuracy oracle $\cO(\by)$ outputs some responses in $V$. Therefore,
  \alequn{\pr_{\by \sim \mu_m^n, \br = \A^{\cO(\by)}} & \lb \acc_\by(\br) \geq \fr{m} + \eps \rb \\
  & \leq    \sum_{\alpha \in V} \pr_{\br = \A^\alpha, \by \sim \mu_m^n}\lb \acc_\by(\br) \geq \fr{m} + \eps \rb \\ & \leq (n+1)^k \cdot e^{-\frac{mn\eps^2}{2+m \eps}} .}

  Now, if $\frac{k \ln(n+1) + \ln(1/\delta)}{n} \geq \frac{1}{m}$ then by definition of $b$ and $\eps$,
  \alequn{\eps &= 2 \max \left\{\sqrt{\frac{b}{nm}},\frac{b}{n} \right\} \\
  &= 2 \frac{b}{n} \geq \frac{2}{m}.} Therefore we obtain that, $\frac{mn\eps^2}{2+m \eps} \geq \frac{n\eps}{2}$ and
  $$(n+1)^k \cdot e^{-\frac{mn\eps^2}{2+m \eps}} \leq e^{k \ln (n+1) - \frac{n\eps}{2}} = e^{\ln \delta} = \delta .$$
  Otherwise (when $\frac{k \ln(n+1) + \ln(1/\delta)}{n} < \frac{1}{m}$) we have that
  $$\eps = 2\sqrt{\frac{b}{nm}} < \frac{2}{m} .$$
  In this case $\frac{mn\eps^2}{2+m \eps} \geq \frac{mn\eps^2}{4}$ and
  $$(n+1)^k \cdot e^{-\frac{mn\eps^2}{2+m \eps}} \leq e^{k \ln (n+1) - \frac{mn\eps^2}{4}} = e^{\ln \delta} = \delta .$$
\end{proof}

\begin{rem}\label{rem:general-upper}
The upper bound applies to arbitrary test set access models that limit the number of bits revealed. Specifically, if the information that the attacker learns about the labeling can be represented using $t$ bits then the same upper bound applies for $b = t +\ln(1/\delta)$. It can also be easily generalized to algorithms whose output has bounded (approximate) max-information with the labeling \cite{DworkFHPRR15:arxiv}.
\end{rem}
This upper bound can also be converted to a simpler one on the expected accuracy by setting $\delta = 1/n$ and noticing that accuracy is bounded above by $1$. Therefore, for
$$\eps =\fr{n}+ 2 \cdot \max \left\{\sqrt{\frac{(k+1)\ln(n+1)}{nm}},\frac{(k+1)\ln(n+1)}{n} \right\},$$
we have $\acc(\A,\mu_m^n) \leq \fr{m} + \eps$.

\mnote{Can we say something about improving not from $1/m$ but some higher accuracy? I thought we could.} \vnote{I think the main connection is that this upper bound applies to the accuracy on part of the dataset where the existing predictors are poor (that is not much better than random). Naturally, if the existing predictor can has a decent prior that may effectively reduce the number of classes. Still there will likely be many classes left so we get a bound on the advantage with reduced number of classes.}

\section{Test set overfitting attacks}

In this section we will examine two attacks that both rely on queries chosen uniformly at random. Our first attack will be a point-wise attack that simply estimates the probability of each of the labels for the point, given the per-query accuracies, and then outputs the most likely label. We will show that this algorithm is optimal among all point-wise algorithms and then analyze the bias of this attack.

We then analyze the accuracy of an attack that relies on access to data points and is not computationally efficient. While such an attack might not be feasible in many scenarios (and we do not evaluate it empirically), it demonstrates the tightness of our upper bound on the optimal bias. This attack is based on exactly reconstructing part of the test set labels. \iffull \else Omitted proofs appear in the supplemental material.\fi

\subsection{Point-wise attack}
\label{sec:pointwise-attack}
The queries in our attack are chosen randomly and uniformly. A point-wise algorithm can implement this easily because each coordinate of such a query is independent of all the rest. Hence we only need to describe how the label of the final classifier on each point is output, given the vector of the point's $k$ labels $\bs = (s_1, \ldots,s_k)$ from each query, and given the corresponding accuracies $\balpha=(\alpha_1,\ldots,\alpha_k)$. To output the label our algorithm computes for each of the possible labels the probability of the observed vector of queries given the observed accuracies. Specifically, if the correct label is $\ell \in [m]$ then the probability of observing $s_j$ given accuracy $\alpha_j$ is $\alpha_j$ if $s_j = \ell$ and $\frac{(1-\alpha_j)}{m-1}$ otherwise. Accordingly, for each label $\ell$ the algorithm considers:
$$\conf(\ell,\bs,\balpha) = \prod_{j\in [k], s_j = \ell } \alpha_j \cdot \prod_{j\in [k], s_j \neq \ell } \frac{(1-\alpha_j)}{m-1} .$$
It then predicts the label that maximizes $\conf$, and in case of ties it picks one of the maximizers randomly.

This algorithm also naturally incorporates the prior distribution over labels $\bpi=(\pi_i)_{i\in [n]}$. Specifically, on point $i$ the algorithm outputs the label that maximizes $\pi_i(\ell) \cdot \conf(\ell,\bs,\balpha)$. Note that the version without a prior is equivalent to one with the uniform prior.
We refer to these versions of the attack algorithm as $\NB$ and $\NB_\bpi$, respectively.
The latter is summarized in Algorithm~\ref{alg:nb-prior}.

\begin{algorithm}[t]
\begin{algorithmic}
\INPUT Query access to a test set of $n$ points over $m$ labels, query budget $k$, and priors $\bpi = (\pi_i)_{i \in [n]}$.
\STATE Draw $k$ queries $Q \in [m]^{n \times k}$ uniformly.
\STATE Submit queries $Q^1, \ldots, Q^k$ and receive corresponding accuracies $\balpha = (\alpha_1, \ldots, \alpha_k)$.
\STATE For $i \in [n]$, compute:
\[
    z_i
    \gets \argmax_{\ell \in [m]} \left\{ 
        \pi_i(\ell)
        \prod_{j \in [k], Q_{i,j} = \ell} \alpha_j
        \prod_{j \in [k], Q_{i,j} \neq \ell} \frac{(1-\alpha_j)}{m-1}
      \right\},
\]
breaking any ties among maximizers uniformly at random.
\OUTPUT Predictions $\bz = (z_1, \dots, z_n)$
\end{algorithmic}
\caption{The $\NB_\bpi$ overfitting attack algorithm.}
\label{alg:nb-prior}
\end{algorithm}

We will start by showing that $\conf(\ell,\bs,\balpha)$ accurately computes the probability of query values.
\begin{lem}
\label{lem:confidence}
Let $\mu_m^{n \times k}$ denote the uniform distribution over $k$ queries. Then for every $\by \in [m]^n$, accuracy vector $\balpha$, $\bs \in [m]^k$, $i\in [n]$ and $j\in[k]$,
$$\pr_{Q} \lb Q_{i,j} = s_j  \cond \acc_\by(Q) = \balpha \rb = \begin{cases} \alpha_j & \mbox{if } s_j = y_i, \\ \frac{1-\alpha_j}{m-1} & \mbox{otherwise.} \end{cases} $$ Further $Q_{i,j}$ are independent conditioned on $\acc_\by(Q) = \balpha$. That is
\alequn{\pr_{Q} &\lb Q_{i} = \bs  \cond \acc_\by(Q) = \balpha \rb \\& = \prod_{j\in [k], s_j = y_i } \alpha_j \cdot \prod_{j\in [k], s_j \neq y_i } \frac{(1-\alpha_j)}{m-1} \\& = \conf(y_i,\bs,\balpha) .}
\end{lem}
\begin{proof}
For every fixed value $\by$, the distribution $Q \sim \mu_m^{n\times k}$ conditioned on $\acc_\by(Q) = \balpha$ is uniform over all query matrices that satisfy $\acc_\by(Q) = \balpha$. This implies that for every $j$ the marginal distribution over $Q^j$ is uniform over the set $\{\bq \cond \acc_\by(\bq) =\alpha_j\}$. We denote this distribution $\rho_{\by,\alpha_j}$. In addition, $Q$ conditioned on $\acc_\by(Q) = \balpha$ is just the product over marginals $\rho_{\by,\alpha_1} \times \cdots \times \rho_{\by,\alpha_k}$.  It is easy to see from the definition of $\rho_{\by,\alpha_j}$, that for every $q \in [m]$,
$$\pr_{\bq \sim \rho_{\by,\alpha_j}} [\bq_i = q] = \begin{cases} \alpha_j & \mbox{if } q = y_i, \\ \frac{1-\alpha_j}{m-1} & \mbox{otherwise.} \end{cases} $$
Thus for every $\bs$,
\alequn{\pr_{Q} &\lb Q_i = \bs  \cond \acc_\by(Q) = \balpha \rb \\& = \prod_{j\in [k], s_j = y_i } \alpha_j \cdot \prod_{j\in [k], s_j \neq y_i } \frac{(1-\alpha_j)}{m-1} \\& = \conf(\ell,\bs,\balpha).
}
\end{proof}

This lemma allows us to conclude that our algorithm is optimal for this setting.
\begin{thm}
\label{thm:nb-optimal}
Let $\bpi = (\pi_1,\ldots,\pi_n)$ be an arbitrary prior on $n$ labels. Let $\A$ be an arbitrary point-wise attack using $k$ randomly and uniformly chosen queries. Then
$$\acc(\A,\bpi) \leq \acc(\NB_\bpi,\bpi) .$$
In particular, $\acc(\A) \leq \acc(\NB)$.
\end{thm}
\begin{proof}
  A point-wise attack $\A$ that uses a query matrix $Q \sim \mu_m^{n \times k}$ is fully specified by some algorithm $\B$ that takes as input the query values for the point $\bs \in [m]^k$ and accuracy values $\balpha =(\alpha_1,\ldots,\alpha_k)$ and outputs a label. By definition,
  \alequn{\acc(\A,\bpi) &=  \E_{\by, Q} \lb \fr{n} \sum_{i\in [n]}  \ind(y_i =\B(Q_i,\acc_\by(Q)) \rb \\
  & = \fr{n} \sum_{i\in [n]} \pr_{\by, Q} \lb y_i =\B(Q_i,\acc_\by(Q))\rb
  ,}
   where $\by \sim \bpi$ and $Q\sim \mu_m^{n\times k}$ (and the same in the rest of the proof).
  Now for every fixed $i\in [n]$,
  \alequn{
    \pr_{\by, Q}
    & \lb y_i = \B(Q_i,\acc_\by(Q)) \rb \\
    & = \sum_{\balpha \in V} \pr_{\by, Q \mid \balpha} \lb y_i = \B(Q_i,\balpha) \rb
        \cdot \pr_{\by, Q}[\acc_\by(Q) = \balpha],
  }
  where by $\by, Q \mid \balpha$ we denote the distribution of $Q$ and $\by$ conditioned on $\acc_\by(Q) = \balpha$ and by $V$ we denote the set of all possible accuracy vectors.
  For every fixed $\balpha \in V$,
   \alequn{&\pr_{\by, Q \mid \balpha} \lb y_i =\B(Q_i,\balpha) \rb  \\&= \sum_{\bs \in [m]^k} \pr_{\by, Q \mid \balpha} \lb y_i =\B(\bs,\balpha)\cond Q_i = \bs \rb \cdot \pr_{\by, Q \mid \balpha}[Q_i = \bs] .
   }
  For every fixed $\bs \in [m]^k$, $\B(\bs,\balpha)$ outputs a random label and the algorithm's randomness is independent of $Q$ and $\by$. Hence,
   \alequn{
     & \pr_{\by, Q \mid \balpha} \lb y_i = \B(\bs,\balpha) \cond Q_i = \bs \rb \\
     & = \sum_{\ell \in [m]} \pr_{\by, Q \mid \balpha} \lb y_i = \ell \cond Q_i = \bs \rb
          \cdot \pr_{\by, Q \mid \balpha} [\B(\bs,\balpha) = \ell] \\
     & \leq \max_{\ell \in [m]} \pr_{\by, Q \mid \balpha} \lb y_i = \ell \cond Q_i = \bs \rb .
   }
   Moreover, the equality is achieved by the algorithm that outputs any value in $$\opt(\bs,\balpha) \doteq \argmax_{\ell \in [m]} \pr_{\by, Q|\balpha} \lb y_i = \ell \cond Q_i = \bs \rb .$$ It remains to verify that $\NB_\bpi$ computes a value in $\opt(\bs,\balpha)$. Applying the Bayes rule we get
   \alequn{ &\pr_{\by, Q|\balpha} \lb y_i = \ell \cond Q_i = \bs \rb \\
   &= \frac{\pr_{\by, Q|\balpha} \lb Q_i = \bs  \cond y_i = \ell \rb \cdot \pr_{\by, Q|\balpha} [ y_i = \ell]}{\pr_{\by, Q|\balpha} [ Q_i = \bs ]}.}

   Now, the denominator is independent of $\ell$ and thus does not affect the definition of $\opt(\bs,\balpha)$. The distribution $Q$ is uniform over all possible queries, and thus for every pair of vectors $\by, \by'$, $$\pr_Q[\acc_\by(Q) = \balpha] = \pr_Q[\acc_{\by'}(Q) = \balpha] .$$ Therefore the marginal of distribution $\by \sim \bpi,Q \sim \mu_m^{n\times k} \cond \acc_\by(Q) = \balpha$ over label vectors is not affected by conditioning. That is, it is equal to $\bpi$. Therefore $$\pr_{\by, Q|\balpha} [ y_i = \ell] = \pr_{\by, Q|\balpha} [ y_i = \ell] = \pi_i(\ell) .$$

   By Lemma \ref{lem:confidence} we obtain that
$$ \pr_{\by, Q|\balpha}  \lb Q_i = \bs  \cond y_i = \ell \rb = \conf(\ell,\bs,\balpha).$$
  This implies that maximizing $\pr_{\by, Q|\balpha} \lb y_i = \ell \cond Q_i = \bs \rb$ is equivalent to maximizing $\pi_i(\ell) \cdot \conf(\ell,\bs,\balpha)$. Hence $\NB_\bpi$ achieves the optimal expected accuracy.

  To obtain the second part of the claim we note that the expected accuracy of $\NB$ does not depend on the target labels $\by$ (the queries and the decision algorithm are invariant to an arbitrary permutation of labels at any point). That is for any $\by,\by' \in [m]^n$, $$\E_{\br=\NB^{\cO(\by)}}[\acc_\by(\br)] =\E_{\br=\NB^{\cO(\by')}}[\acc_{\by'}(\br)] .$$
  This means that the worst case accuracy of $\NB$ is the same as its average-case accuracy for labels drawn from the uniform distribution $\mu_m^n$. In addition, $\NB_{\mu_m^n}$ is equivalent to $\NB$. Therefore,
  $$\acc(\A) \leq \acc(\A,\mu_m^n) \leq \acc(\NB_{\mu_m^n},\mu_m^n) = \acc(\NB).$$
\end{proof}

We now provide the analysis of a lower bound on the bias achieved by $\NB$. Our analysis will apply to a simpler algorithm that effectively computes the plurality label among those for which accuracy is sufficiently high (larger than the mean plus one standard deviation). Further, to simplify the analysis, we take the number of queries to be a draw from the Poisson distribution. This Poissonization step ensures that the counts of the times each label occurs are independent. The optimality of the $\NB$ attack implies that the bias achieved by $\NB$ is at least as large as that of this simpler attack.

The key to our proof of Theorem \ref{thm:nb-analysis} is the following lemma about biased and Poissonized multinomial random variables \iffull that we prove in Appendix \ref{sec:poisson}\fi.
\begin{lem}
\label{lem:plurality}
For $\gamma \geq 0$ let $\rho_\gamma$ denote the categorical distribution $\rho_\gamma$ over $[m]$ such that $\pr_{s \sim \rho_\gamma}[s = m] = \fr{m} + \gamma$ and for all $y \neq m$, $\pr_{s \sim \rho_\gamma}[s = y] = \fr{m} - \frac{\gamma}{m-1}$. For an integer $t$, let $\mnom(t,\rho_\gamma)$ be the multinomial distribution over counts corresponding to $t$ independent draws from $\rho_\gamma$. For a vector of counts $\bc$, let $\argmax(\bc)$ denote the index of the largest value in $\bc$. If several values achieve the maximum then one of the indices is picked randomly.
Then for $\lambda \geq 2 m \ln (4m)$ and $\gamma \leq \frac{1}{8\sqrt{\lambda m}}$,
$$
\pr_{t \sim \Pois(\lambda), \bc \sim \mnom(t,\rho_\gamma)} [\argmax(\bc) = m]
    \geq \fr{m} + \Omega \lp \frac{\gamma \sqrt{\lambda}}{\sqrt{m}} \rp
$$
\end{lem}

Given this lemma the rest of the analysis follows quite easily.
\begin{thm}
\label{thm:nb-analysis}
For any $m \geq 2$, $n \geq k \geq k_{\min} = O(\ln n+ m \ln m)$, we have that
$$\acc(\NB) = \fr{m} + \Omega\lp\frac{\sqrt{k}}{m\sqrt{n}}\rp .$$
\end{thm}
\begin{proof}
Let $\gamma = \frac{\sqrt{1-1/m}}{3\sqrt{mn}}$ and we consider a point-wise attack algorithm $\B$ that given a vector $\bs \in [m]^k$  of query values at a point and a vector $\balpha$ of accuracies computes the set of indices $J \subseteq [k]$, where $\alpha_j \geq \fr{m} +\gamma$. We denote $t =|J|$. The algorithm then samples $v$ from $\Pois(\lambda)$ for $\lambda=k/8$. If $v \leq t$ then let $J'$ denote the first $v$ elements in $J$, otherwise we let $J'=J$. $\B$ outputs the plurality label of labels in $\bs_{J'} = (s_j)_{j\in J'}$.

To analyze the algorithm, we denote the distribution over $\bs$, conditioned on the accuracy vector being $\balpha$ and correct label of the point being $y$ by $\rho(\balpha,y)$. Our goal is to lower bound the success probability of $\B$ $$\pr_{\bs \sim \rho(\balpha,y)}[\B(\bs,\balpha) = y] .$$

Lemma \ref{lem:confidence} implies that elements of $\bs$ are independent and for every $j\in[k]$, $s_j$ is equal to $y$ with probability $\alpha_j$ and $\frac{1-\alpha_j}{m-1}$, otherwise. Therefore for every $j \in J$, $s_j$ is biased by at least $\gamma$ towards the correct label $y$. We will further assume that $s_j$ is biased by exactly $\gamma$ since larger bias can only increase the success probability of $\B$.

Now let $\delta = \pr[v> t]$. The distribution of $|J'|$ is $\delta$ close in total variation distance to $\Pois(\lambda)$. By Lemma \ref{lem:plurality}, this means that
\equ{\pr[\pl(\bs_{J'}) = y] \geq \fr{m} + \Omega\lp \frac{\sqrt{k}\gamma}{\sqrt{m}}\rp - \delta = \fr{m} + \Omega\lp \frac{\sqrt{k}}{m\sqrt{n}}\rp - \delta,\label{eq:simple-adv-bound}} where we used the assumptions $k \geq k_{\min}$ and $n \geq k$ to ensure that the conditions $\lambda \geq 2 m \ln (4m)$ and $\gamma \leq \frac{1}{8\sqrt{\lambda m}}$ hold.

Hence to obtain our result it remains to estimate $\delta$. We view $t$ as jointly distributed with $\balpha$. Let $\phi$ denote the distribution of $\balpha$ for $Q \sim \mu_n^{n \times k}$ and any vector $\by$.
For every $j \in [k]$, $\alpha_j$ is distributed according to the binomial distribution $\Bin(n,1/m)$. By using the Berry-Esseen theorem (Lemma \ref{lem:berry-esseen}), we obtain that $$\pr\lb \alpha_j  \geq \fr{m} + \frac{\sigma}{3} \rb \geq \pr\lb \zeta \geq \frac{\sigma}{3}\rb - O\lp\sqrt{\frac{m}{n}}\rp ,$$ where $\sigma^2 = \frac{1-1/m}{mn}$ and $\zeta$ is normally distributed with mean 0 and variance $\sigma^2$. In particular, for sufficiently large $n$,
$$\pr\lb \alpha_j  \geq \fr{m} + \gamma \rb \geq 1/3 .$$

Now by Chernoff bound (Lemma \ref{lem:chernoff}), we obtain that for sufficiently large $k$, $$\pr\lb t \leq \frac{k}{4} \rb \leq e^{-k/96}.$$
In addition, by the concentration of $\Pois(k/8)$ (Lemma \ref{lem:poisson-concetrate}) we obtain that
$$\pr \lb v \geq \frac{k}{4}\rb \leq e^{-k/32} .$$

Therefore, by the union bound, $\delta \leq e^{-k/96} + e^{-k/32}$ and thus for $k \geq k_{\min}$ we will have that $\delta = o(1/n) = o(\sqrt{k}/(m\sqrt{n}))$. Plugging this into eq.~\eqref{eq:simple-adv-bound} we obtain the claim.
\end{proof}

\subsection{Reconstruction-based attack}
Our second attack relies on a probabilistic argument, showing that any dataset's label vector is, with high probability, uniquely identified by the accuracies of $O\lp\max\left\{ \frac{n\ln m}{\ln(n/m)}, m\ln(nm)\right\}\rp$ uniformly random queries. This argument was first used for the binary label case by \citet{ErdosR63} and generalized to arbitrary $m$ by \citet{Chvatal83}. We further generalize it to allow identification when the accuracy values are known only up to a fixed shift. This is needed as we apply this algorithm to a subset of labels such that the accuracy on the remaining labels is unknown. Formally, the unique identification property follows.
\begin{thm}
\label{thm:indentify}
 Say that a query matrix $Q \in [m]^{n\times k}$ recovers any label vector from shifted accuracies if there do not exist distinct $\by,\by'\in[m]^n$ and shift $\beta \in \R$ such that
 $$ \acc_\by(Q) = \acc_{\by'}(Q) + \beta \cdot (1,1,\ldots,1). $$
 For $m\geq 3$ and $k = \max\left\{ \frac{5n\ln m}{\ln(n/4m)}, 20 m\ln(nm)\right\}$, with probability at least $1/2$ over the choice of random $Q \sim \mu_m^{n\times k}$, $Q$ recovers any label vector from shifted accuracies.
\end{thm}
\begin{proof}
Let $\by\neq\by'\in[m]^n$ be an arbitrary pair of indices. We describe the difference between $\by$ and $\by'$ using the set of indices where they differ $I=\Delta(\by,\by') = \{i \cond y_i \neq y'_i\}$ and the vectors restricted to this set $\by_I = (y_i)_{i\in I}$ and $\by'_I = (y'_i)_{i\in I}$.
It is easy to see from the definition that for any query $\bq$,
$$ \acc_{\by}(\bq) - \acc_{\by'}(\bq) = \fr{n} \sum_{i\in I} \lp \ind(y_i = q_i) -  \ind(y'_i = q_i) \rp .$$
In particular, the difference is fully determined by $I=\Delta(\by,\by')$, $\by_I$ and $\by'_I$.

This implies that for a randomly chosen $\bq \sim \mu_m^n$, $\acc_{\by}(\bq) - \acc_{\by'}(\bq)$ is distributed as a sum of $w=|I|$ independent random variables from distribution that is equal $1/n$ with probability $1/m$, $-1/n$ with probability $1/m$ and $0$ otherwise. Equivalently, this distribution can be seen as a sum $$\fr{n} \sum_{i \in [w]} b_i \sigma_i ,$$ where each $b_i$ is independent Bernoulli random variable with bias $2/m$ and each $\sigma_i$ is an independent Rademacher random variable. We use $v$ to denote the random variable
$$ v = \sum_{i \in [w]} b_i \sigma_i $$ and let $b$ denote the jointly distributed value $$b = \sum_{i \in [w]} b_i .$$

We first deal with shift $\beta = 0$. For this we will first need to upper-bound the probability $p_w \doteq \pr[v=0]$.
Note that conditioned on $b = j$, $v$ is distributed as sum of $j$ Rademacher random variables. Standard bounds on the central binomial coefficient imply that for even $j \geq 2$, $$\pr[v = 0 \cond b = j] \leq \frac{1}{\sqrt{j}}$$ and for odd $j$, $\pr[v = 0 \cond b = j] = 0$. In particular, for all $j \geq 1$, $\pr[v = 0 \cond b = j] \leq 1/2$.

This gives us that
\alequ{\pr[v = 0] & \leq \pr[b = 0] + \fr{2}\pr [b > 1] \nonumber
\\&= \fr{2} + \fr{2} \lp 1-\frac{2}{m} \rp^w \nonumber
\\& \leq \fr{2} + \fr{2} e^{-\frac{2w}{m}}.  \label{eq:small-w}}

Now using the multiplicative Chernoff bound we get that
$$\pr \lb b \leq \frac{w}{m} \rb \leq e^{-\frac{w}{6m}} .$$
This implies that
\alequ{\pr[v = 0] & \leq \pr\lb b < \frac{w}{m}\rb +  \sqrt{\frac{m}{w}} \pr\lb b \geq \frac{w}{m}\rb \nonumber \\&= e^{-\frac{w}{6m}} + \sqrt{\frac{m}{w}} . \label{eq:large-w}}

Given a matrix $Q$ of $k$ randomly and independently chosen queries we have
$$ \pr_{Q \sim \mu_m^{n\times k}}\lb \acc_\by(Q) = \acc_{\by'}(Q) \rb \leq p_w^k .$$
There are at most ${n \choose w} m^{2w}$ possible differences between a pair of vectors $\by,\by'$.
Therefore, by the union bound for every $w$, probability that there exists a pair of vectors $\by,\by'$ that differ in $w$ positions and for which the accuracies on all $k$ queries are identical is at most $$ {n \choose w} m^{2w} \cdot p_w^k .$$

If $1 \leq w < 2m$ then eq.~\eqref{eq:small-w} implies that $$p_w \leq \fr{2} + \fr{2} e^{-\frac{2w}{m}} \leq \fr{2} + \fr{2}\lp 1-\frac{w}{m}\rp \leq e^{-\frac{w}{2m}} $$ and our union bound is
\alequn{{n \choose w} m^{2w}& \cdot e^{-\frac{kw}{2m}} \\
&\leq\lp \frac{n e m^2}{w}\rp^w  \cdot e^{-\frac{kw}{2m}} \\
&\leq e^{w \ln(enm^2) - \frac{kw}{2m}} \\
& \leq \lp \frac{1}{4n^2} \rp^w \leq \frac{1}{2n^2}
,}
where we used the condition that $$k \geq 20m\ln(nm) \geq 2m\ln(2en^3m^2).$$

If $2m \leq w < 6m$ then eq.~\eqref{eq:small-w} implies that $$p_w \leq \fr{2} + \fr{2} e^{-\frac{2w}{m}} \leq \fr{2} + \fr{2} e^{-1} \leq e^{-1/3} $$ and our union bound is
\alequn{{n \choose w} m^{2w}& \cdot e^{-\frac{k}{3}} \\
&\leq e^{w \ln(enm/2) - \frac{k}{3}} .}
This bound is maximized for $w =6m$ giving  $e^{6m \ln(nm) - \frac{k}{3}}$.
Using the condition $$k \geq 20 m\ln(nm) \geq 18m\ln(enm/2) + 3\ln(2n^2)$$ we get
an upper bound of $\frac{1}{2n^2}$.

If $w \geq 6m$ then eq.~\eqref{eq:large-w} implies that $p_w \leq \sqrt{\frac{m}{4w}}$ and our union bound is
  $${n \choose w} m^{2w} \cdot \lp \sqrt{\frac{m}{4w}}\rp^k .$$ This bound is maximized for $w=n$, which gives an upper bound of $$ m^{2w} \cdot \lp\frac{n}{4m}\rp^{k/2} \leq \frac{1}{2n^2} ,$$ where we used the condition that $$k \geq \frac{5n\ln m}{\ln(n/4m)} \geq 2\frac{2n\ln(m) + \ln(2n^2)}{\ln(n/4m)} .$$

Now by using a union bound over all values of $w \in [n]$ we get that probability that there exist distinct $\by,\by'\in[m]^n$ such that
 $\acc_\by(Q) = \acc_{\by'}(Q)$ is at most $1/(2n)$. Now to deal with any other $\beta \in \{1/n,\ldots,1\}$ (we only need to treat positive $\beta$s since the definition is symmetric) we observe that for $m \geq 3$ and any $w$, $$\pr[v = n \beta ] \leq \pr[v = 0] =p_w .$$
 By the same argument this implies that probability that there exist distinct $\by,\by'\in[m]^n$ such that
 $$\acc_\by(Q) = \acc_{\by'}(Q) + \beta \cdot (1,1,\ldots,1)$$ is at most $1/(2n)$. Taking the union bound over all values of $\beta$ we obtain the claim.
\end{proof}

\begin{algorithm}[t]
\begin{algorithmic}
\INPUT Query access to a test set of $n$ points over $m$ labels, example budget $t \le n$.
\STATE Draw $k$ queries $R \in [m]^{t \times k}$ uniformly over $[m]^{t \times k}$.
\STATE Let $Q \in [m]^{n \times k}$ be the matrix that extends $R$ by appending $n-t$ rows of ones.
\STATE Submit queries $Q^1, \dots, Q^k$ and receive corresponding accuracies $\balpha = (\alpha_1, \dots, \alpha_k)$
\STATE Compute $\bz = (z_1, \dots, z_n) \in [m]^n$ as any vector satisfying $\acc_{\bz}(Q) = \balpha$.
\STATE Draw random predictions $z'_1, \dots, z'_{n-t}$ uniformly over $[m]^{n-t}$.
\OUTPUT Predictions $(z_1, \dots, z_t, z'_1, \dots, z'_{n-t})$
\end{algorithmic}
\caption{The reconstruction-based overfitting attack algorithm.}
\label{alg:recon}
\end{algorithm}

Naturally, if for all distinct labeling $\by,\by'$, $\acc_\by(Q) \neq \acc_{\by'}(Q)$ then we can recover the unknown labeling $\by$ simply by trying out all possible labeling $\by'$ and picking the one for which the $\acc_\by(Q) = \acc_\by'(Q)$. Thus an immediate implication of Thm.~\ref{thm:indentify} is that there exists a fixed set of $k=O\lp\max\left\{ \frac{n\ln m}{\ln(n/m)}, m\ln(nm)\right\}\rp$ queries that can be used to reconstruct the labels. In particular, this gives an attack algorithm with accuracy $1$. If $k$ is not sufficiently large for reconstructing the entire set of labels then it can be used to reconstruct a sufficiently small subset of the labels (and predict the rest randomly). Hence we obtain the following bound on achievable bias.
\begin{cor}
\label{cor:reconstruct}
For any $k\geq 40 m\ln(m)$, there exists an attack $\A$ with access to points such that $$\acc(\A) = \min\left\{1, \fr{m} + \Omega\lp \frac{ k \ln(k/m)}{n \ln m} \rp \right\} .$$
\end{cor}
\begin{proof}
We first let $t$ be the largest value for which Thm.~\ref{thm:indentify} guarantees existence of a set of queries of size $k$ that allows to fully recover $t$ labels from shifted accuracies. Using the bound from Thm.~\ref{thm:indentify} we get that $t = \Omega\lp \frac{ k \ln(k/m)}{ \ln m} \rp$. If $t \geq n$ then we recover the labels and output them. Otherwise, let $R \in [m]^{t \times k}$ be the set of queries that recovers $t$ labels and let $\by_{[t]}$ be the first $t$ values of $\by$. We extend $R$ to a set $Q$ of queries over $n$ labels by appending a fixed query $(1,1,\ldots,1)$ over the remaining $n-t$ coordinates.

Now to recover $\by_{[t]}$ we need to observe that, if there exists a vector $\bz \in [m]^t$ such that
$$t \cdot \acc_{\bz}(R) = n \cdot \acc_{\by}(Q) - (n-t) \beta(1,1,\ldots,1)$$
for some fixed value $\beta$, then $\by_{[t]} =\bz$. This follows from the fact that
$$t \cdot \acc_{\by_{[t]}}(R) = n \cdot \acc_{\by}(Q) - (n-t) \beta' (1,1,\ldots,1),$$
where $\beta'$ is the accuracy of all $1$ labels on the last $n-t$ coordinates of $\by$.
This implies that
$$\acc_{\bz}(R) = \acc_{\by_{[t]}}(R) + \lp \frac {n - t} t \rp (\beta'-\beta)(1,1,\ldots,1).$$
By the property of $R$ this implies that $\bz = \by_{[t]}$. Having found $\bz = \by_{[t]}$ we output a labeling that is equal to $\bz$ on the first $t$ labels and is random and uniform over the rest. The expected accuracy of this labeling is
$$\frac{t}{n} + \fr{m} \lp 1 - \frac{t}{n} \rp = \fr{m} + \frac{t}{n} \frac{m-1}{m}  = \fr{m} +  \Omega\lp \frac{ k \ln(k/m)}{n \ln m} \rp.$$
\end{proof}

The resulting reconstruction-based attack is summarized as Algorithm~\ref{alg:recon}.

\section{Experimental evaluation}

\begin{figure}[t]
    \centering
    \subfigure[Full test set]{
        \includegraphics[width=.48\columnwidth]{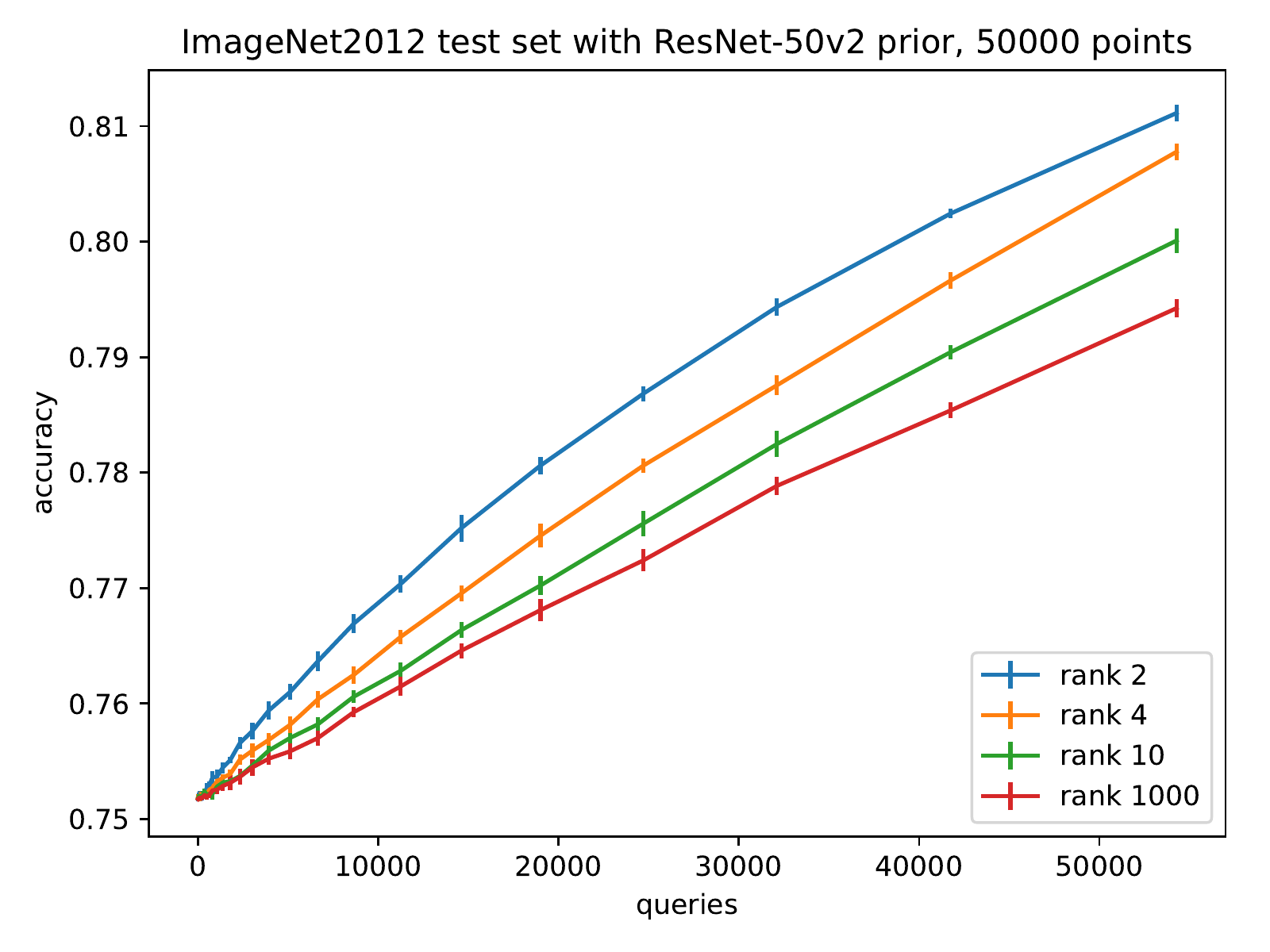}}
    \subfigure[10K least confident points]{
        \includegraphics[width=.48\columnwidth]{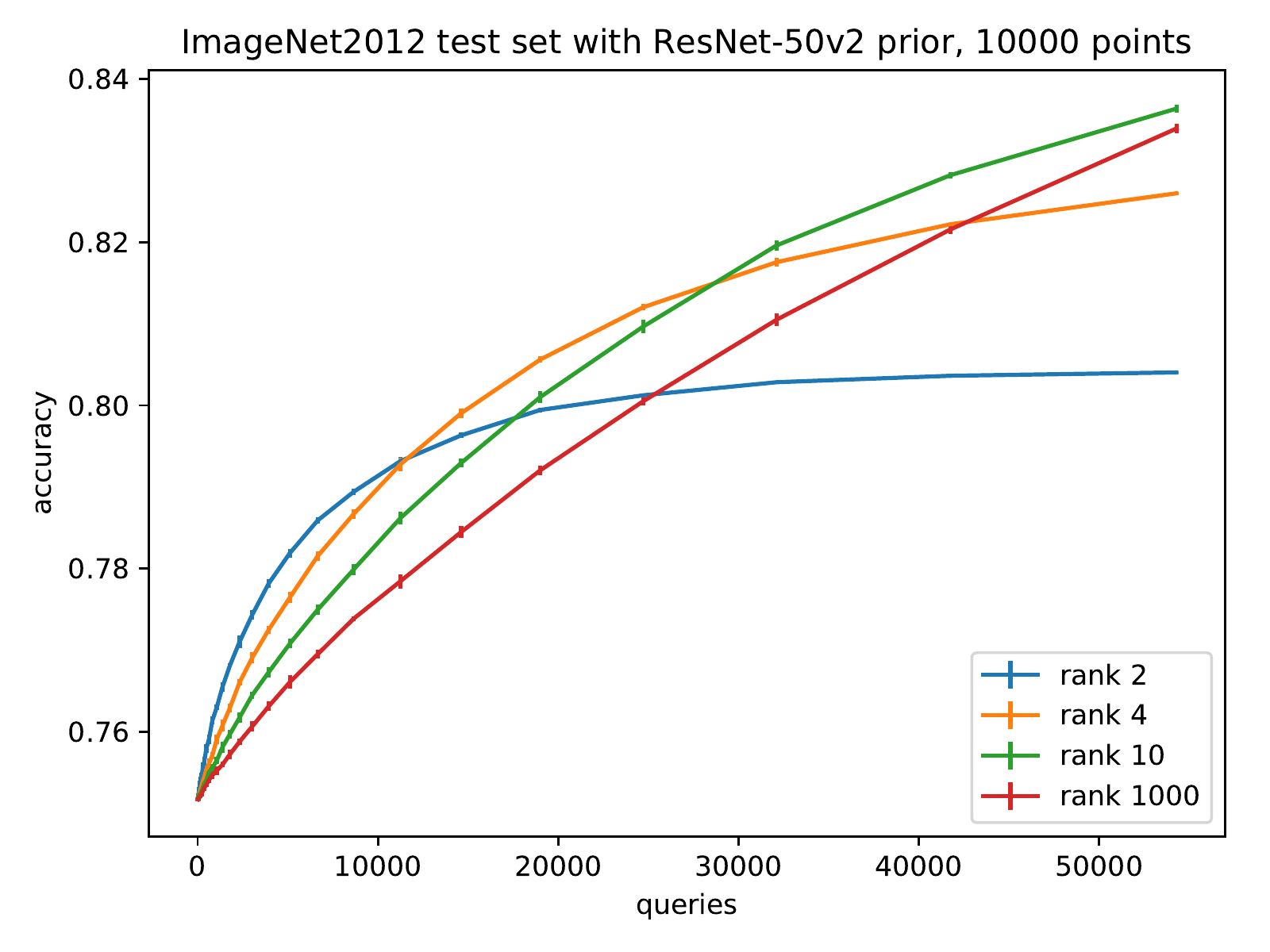}}
    \caption{Average accuracy (with standard deviation bars), over 10 attack trials, of the $\NB_\bpi$ attack against the ImageNet test set. The attacker's gains improve when the effective class count, as indicated by \emph{rank} (the value $R$ used in the top-$R$ heuristic) is reduced, illustrating the increasing vulnerability of the test set when classes are removed.}
    \label{fig:imgnet-all}
\end{figure}

\iffull
\iffull \else \section{Additional figures} \fi
\label{sec:app-figures}
\begin{figure*}[t]
    \centering
    \subfigure[Full test set]{
        \includegraphics[width=.8\textwidth]{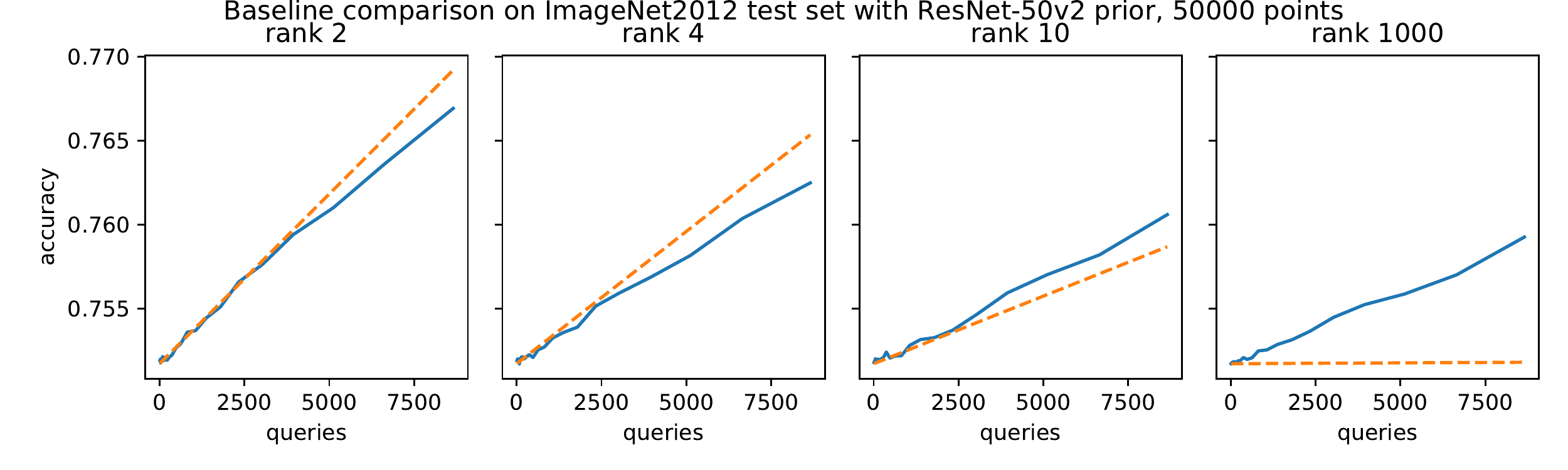}}
    \subfigure[10K least confident points]{
        \includegraphics[width=.8\textwidth]{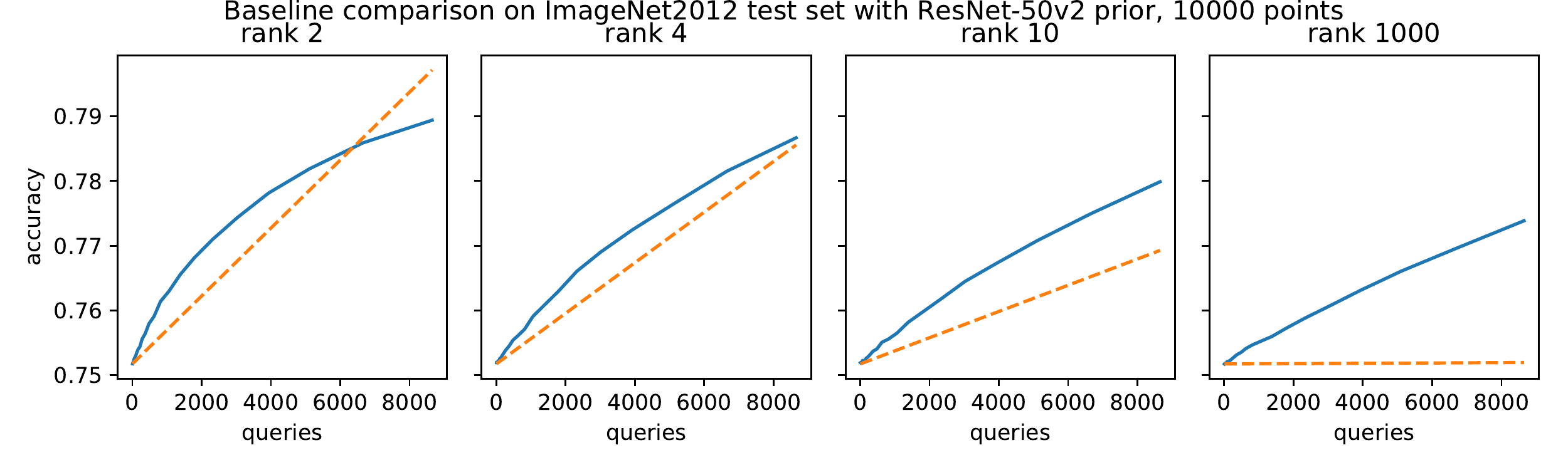}}
    \caption{The average accuracies depicted in Figure~\ref{fig:imgnet-all} in comparison with an analytical baseline: the expected performance of the linear scan attack at the same number of queries.}
    \label{fig:imgnet-baseline}
\end{figure*}

\begin{figure*}[t]
    \centering
    \subfigure[10K random points]{
        \includegraphics[width=.48\columnwidth]{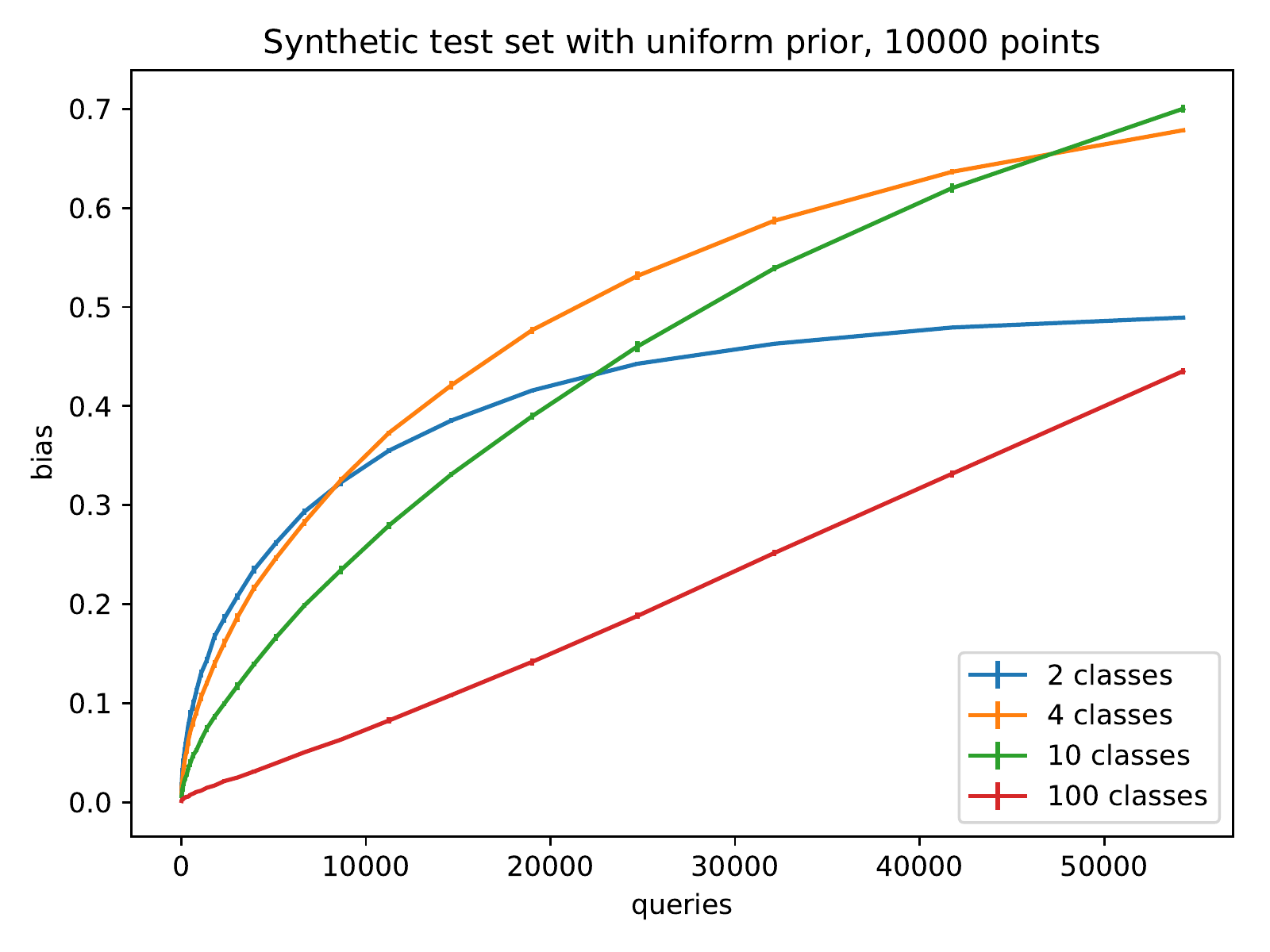}}
    \subfigure[50K random points]{
        \includegraphics[width=.48\columnwidth]{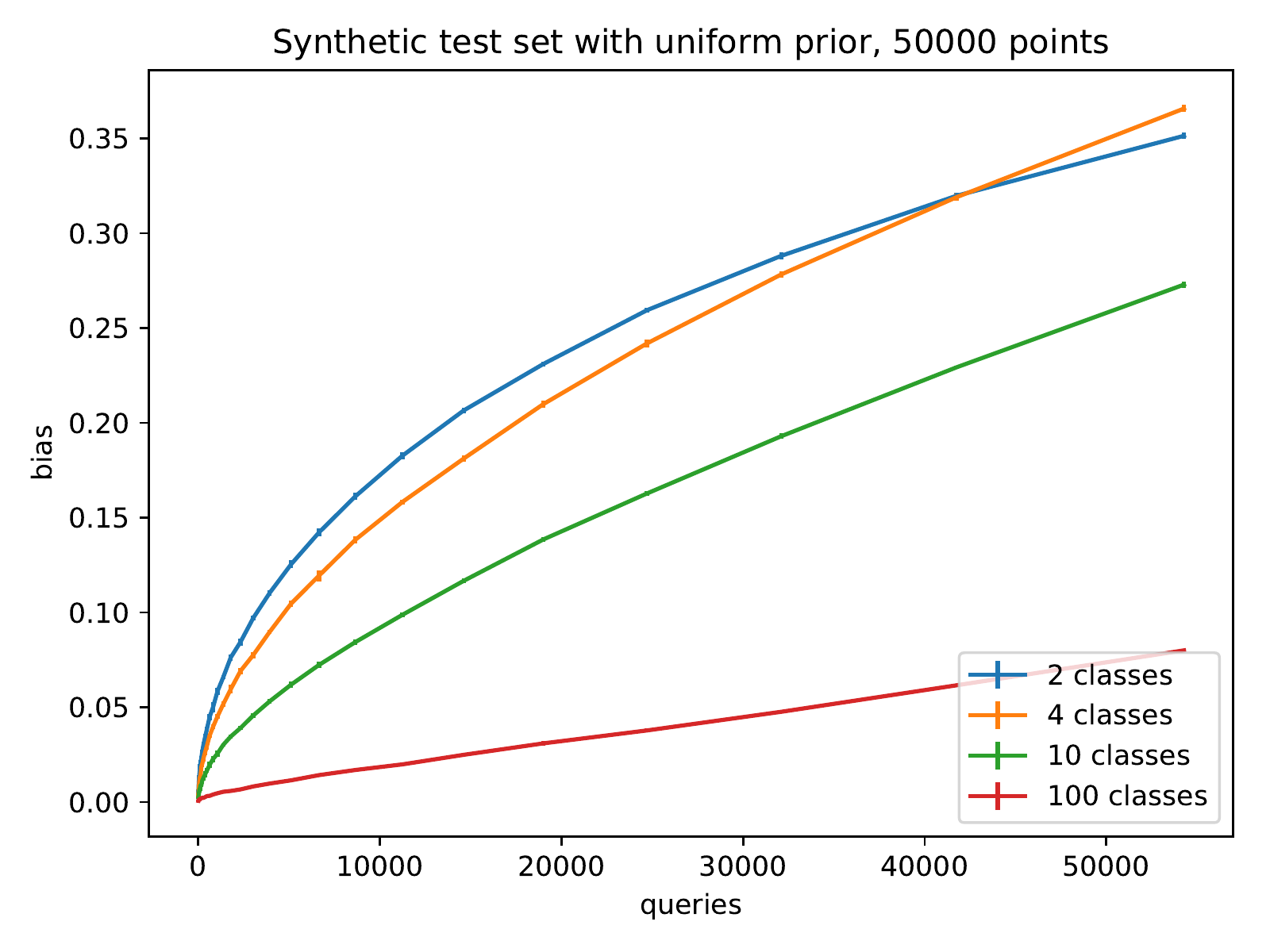}}
    \caption{Average bias (with negligible standard deviation bars), over 10 attack trials, of the $\NB$ attack against uniformly random test sets comprising different class counts. Note that the maximum achievable bias, which is $1-1/m$ for $m$ classes, differs for each curve.}
    \label{fig:synth}
\end{figure*}

\begin{figure*}[t]
    \includegraphics{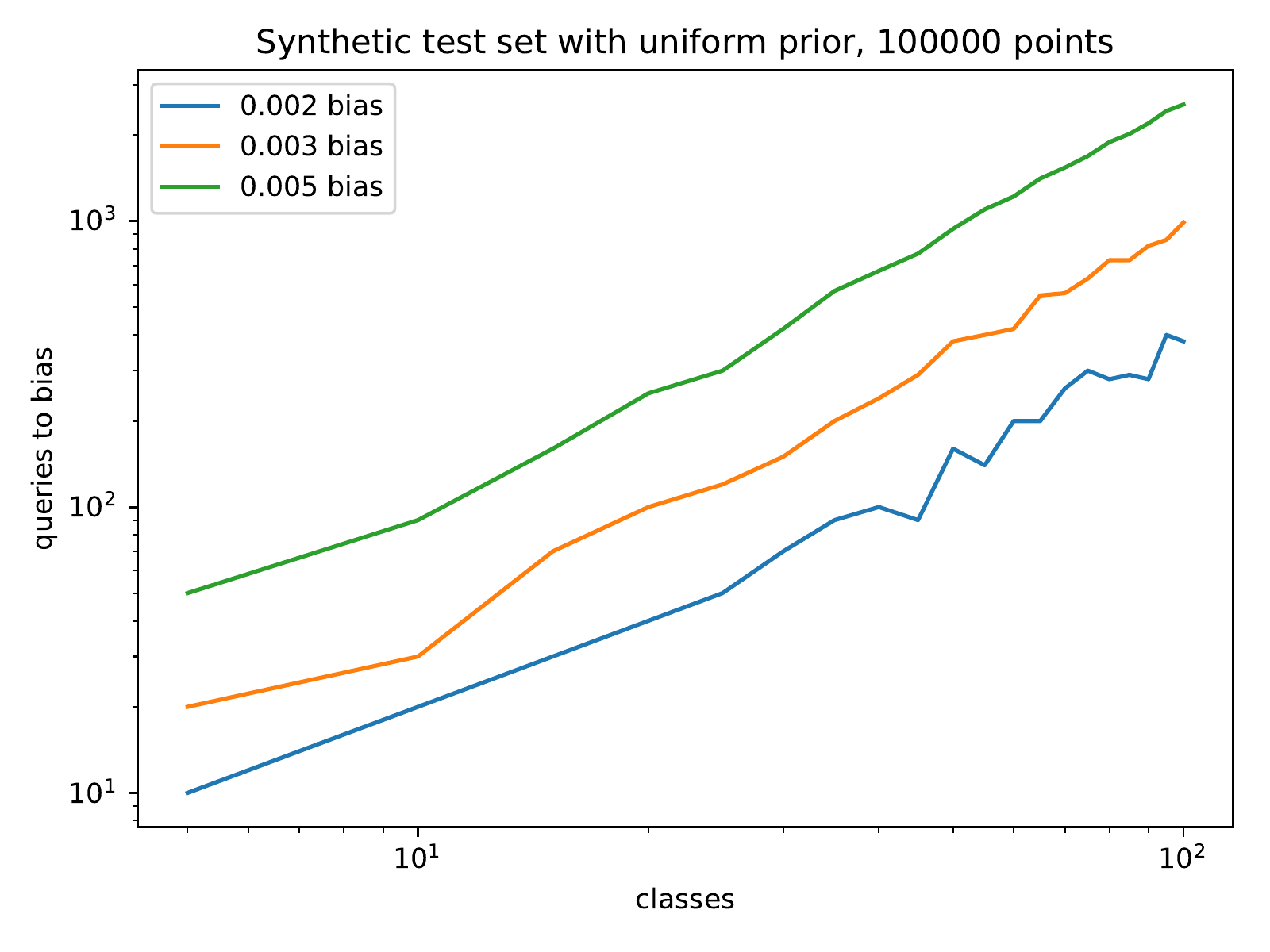}
    \caption{The number of queries at which a fixed advantage over $1/m$ is first attained, while the number of class labels $m$ varies, on a randomly generated test set of size 100{,}000.
    The endpoints of the curves form slopes (under the log-log axis scaling) of roughly $1.2$ (for the $0.002$ bias curve) and $1.3$ (for the other curves), suggesting that, to attain a fixed bias, the number of queries $k$ must indeed grow superlinearly with $m$, as supported by the bound in Theorem~\ref{thm:nb-analysis}.}
    \label{fig:synth-bias}
\end{figure*}

\begin{figure*}[t]
    \includegraphics{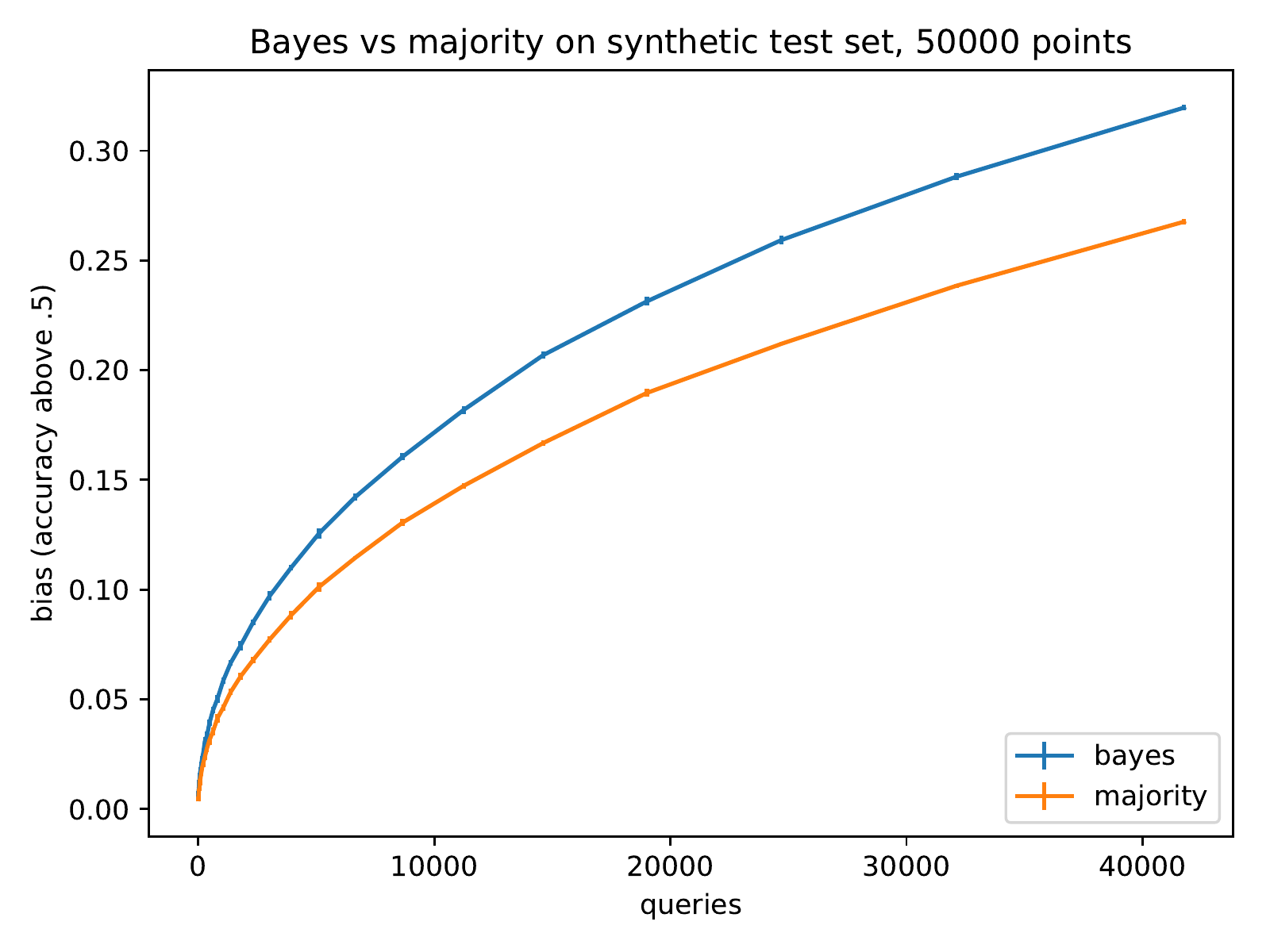}
    \caption{Average bias (with negligible standard deviation bars), over 10 attack trials, of two attacks---$\NB$ and the majority attack of \citet{BlumH15}---against uniformly random binary-labeled test sets.}
    \label{fig:comparison-to-majority}
\end{figure*}

\fi

This section presents a variety of experiments intended (i) to corroborate formal bounds, (ii) to provide a comparison to previous attack in the binary classification setting, and (iii) to explore the practical application of the $\NB$ attack from Section~\ref{sec:pointwise-attack}.

To visualize the attack's performance, we first simply simulate our attack directly on a test set of labels generated uniformly at random from $m$ classes. The attack assumes the uniform prior over the same labels.
Figure~\ref{fig:synth} \iffull \else in Appendix~\ref{sec:app-figures} \fi shows the observed advantage of the attack over the population error rate of $1/m$, across a range of query budgets, on test sets of size 10{,}000 and 50{,}000 respectively.\footnote{The number of points in these synthetic test sets is chosen to mirror the CIFAR-10 and ImageNet test sets.}

Figure~\ref{fig:synth-bias} \iffull \else in Appendix~\ref{sec:app-figures} \fi shows the number of queries at which a fixed advantage over $1/m$ is first attained, while the number of class labels $m$ varies, on a test set of size 100{,}000. To maintain a fixed value of the bound in Theorem~\ref{thm:nb-analysis}, an increase in the number of classes $m$ requires a quadratic increase in the number of queries $k$. The endpoints of the curves in Figure~\ref{fig:synth-bias}, on a log-log scale, form lines of slope greater than 1, supporting the conjecture that, to attain a fixed bias, the number of queries $k$ grows superlinearly with $m$.

In the binary classification setting, we compare to the majority-based attack proposed by \citet{BlumH15}, under the same synthetic dataset. Recall that the $\NB$ attack is based on a majority (more generally, plurality) weighted by the per-query accuracies.
The majority function is weighted only by $\pm 1$ values, as a means of ensuring non-negative correlation of each query with the test set labels. It does not consider low- and high-accuracy queries differently, where $\NB$ does.
Figure~\ref{fig:comparison-to-majority} \iffull \else in Appendix~\ref{sec:app-figures} \fi shows the observed relative advantage of the $\NB$ attack.
Note that simulating uniformly random binary labels places both attacks on similar starting grounds: the attacks otherwise differ in that $\NB_\bpi$ can incorporate a prior distribution $\bpi$ over class labels to its advantage.

Our remaining experiments aim to overfit to the ImageNet test set associated with the 2012 ILSVRC benchmark.
As a form of prior information, we incorporate the availability of a standard and (nearly) state of the art model.
Specifically, we train a ResNet-50v2 model over the ImageNet training set.
On the test set, this model achieves a prediction accuracy of $75.1\%$
and a top-$R$ accuracy of $85.3\%$, $91.0\%$, and $95.3\%$ for $R = 2, 4$, and $10$, respectively.

As is common practice in classification, the ResNet model is trained under the cross-entropy loss (a.k.a.\ the multiclass logistic loss). That is, it is trained to output scores (logits) that define a probability distribution over classes, from which it predicts the maximally-probable class label.
We use the model's logits---a 50{,}000 by 1000 array---as the sole source of side information for attack. All results are summarized in Figure~\ref{fig:imgnet-all}, several highlights of which follow.

First, we consider plugging the model's predictive distribution in as the prior $\bpi$ in the $\NB_\bpi$ attack, yielding modest gains, e.g.\ a $0.42\%$ accuracy boost after 5200 queries (averaged over 10 simulations of the attack).

Next, we observe that the model is highly confident about many of its predictions. Recalling the dependence on the test set size $n$ in our upper bound, we consider a simple heuristic for culling points. Namely, we select the 10K points for which the model is least confident of its prediction in order to attack a test set that is a fifth of the original size. This heuristic presents a trade-off: one reduces $n$ to 10K, but commits to leaving intact the errors made by the model on the 40K more confident points. Applying this heuristic improves gains further, e.g.\ to a $1.44\%$ accuracy boost after 5200 queries.

Finally, we consider another heuristic to reduce $m$, the effective number of classes in the attack, per this paper's focus on the multiple class count. Observing that the model has a high top-$R$ accuracy (i.e.\ recall at $R$) for relatively small values of $R$, it is straightforward to apply the $\NB_\bpi$ attack not to the original classes, but to selecting (pointwise) which of the model's top-$R$ predictions to take. This heuristic presents a trade-off as well: one reduces $m$ down to $R$, but commits to perform no better than the top-$R$ accuracy of the model, a quantity that increases with $R$. Applying this heuristic together with the previous improves the attacker's advantage further. For instance, at $R=2$, we observe a $3.0\%$ accuracy boost after 5200 queries.

To put these numbers in perspective, we compare to a straightforward analytical baseline in \iffull Figure~\ref{fig:imgnet-baseline}\else supplemental material\fi: the expected performance of the ``linear scan attack.''
Namely, this is an attack that begins with a random query vector and successively submits queries by modifying the label of one point at a time, discovering the label's true value whenever the observed test set accuracy increases.

\section*{Acknowledgements}
We thank Cl\'{e}ment Canonne for his suggestion to use Poissonization in the proof of Theorem~\ref{thm:nb-analysis}.
We thank Chiyuan Zhang for his crucial help in the setup of our ImageNet experiment.
We thank Kunal Talwar, Tomer Koren, and Yoram Singer for insightful discussion.

\iffull
\printbibliography
\else
\bibliographystyle{icml2019}
\bibliography{vf-allrefs-local,holdout}
\fi

\newpage
\appendix
\iffull
\section{Proof of Lemma \ref{lem:plurality}}
\label{sec:poisson}
We start with some definitions and properties of the Poisson distribution that we will need in the proof.

A Poisson random variable $V$, with parameter $\lambda$, is the random variable that for all non-negative integers $t$, satisfies $\pr[V = t]=e^{-\lambda} \frac{\lambda^t}{t!}$. We denote its density by $\Pois(\lambda)$.  For $U\sim \Pois(\lambda_1)$ and $U\sim \Pois(\lambda_2)$, $U+V$ is distributed according to $\Pois(\lambda_1+\lambda_2)$.

We will use the following result referred to as Poissonization of a multinomial random variable.
\begin{fact}
\label{fac:poissonize}
Let $\rho(\bp)$ be a categorical distribution over $[m]$ defined by a vector of probabilities $\bp = (p_1,\ldots,p_m)$ and let $\mnom(k,\bp)$ be the multinomial distribution over counts corresponding to $k$ independent draws from $\rho(\bp)$. Then for any $\lambda >0$ and $V \sim \Pois(\lambda)$ we have that $\mnom(V,\bp)$ is distributed as $$\Pois(p_1\lambda) \times \Pois(p_2\lambda)\times \cdots \times \Pois(p_m\lambda).$$
\end{fact}

We will need a relatively tight bound on the concentration of a Poisson random variable. Its simple proof can be found, for example, in a note by\citet{Canonne:2017note}.
\begin{lem}[\citep{Canonne:2017note}]
\label{lem:poisson-concetrate}
For any $\lambda>0,x\geq 0$,
$$\pr_{V\sim \Pois(\lambda)}[V \geq \lambda + x] \leq e^{-(\lambda+x)\ln\lp 1+\frac{x}{\lambda}\rp-x} \mbox{ and}$$
$$\pr_{V\sim \Pois(\lambda)}[V \leq \lambda - x] \leq e^{-(\lambda-x)\ln\lp 1-\frac{x}{\lambda}\rp-x} .$$
In particular,
$$\pr_{V\sim \Pois(\lambda)}[V \geq \lambda + x] \leq e^{\frac{-x^2}{2(\lambda+x)}}  \mbox{ and}$$
$$\pr_{V\sim \Pois(\lambda)}[V \leq \lambda - x] \leq e^{\frac{-x^2}{2(\lambda+x)}} .$$
\end{lem}

Using this concentration inequality we show that the density of the Poisson random variable can be related in a tight way to the corresponding tail probability.
\begin{lem}
\label{lem:poisson-tail-density}
For any $\lambda > 0$ and integer $t \geq 0$ and $x = |t-\lambda|$,
$$\pr_{V \sim \Pois(\lambda)}[V = t] \geq \frac{e^{-t \ln\lp \frac{t}{\lambda}\rp-x}}{e \sqrt{t}} .$$
In particular, for $t\geq \lambda$,
$$\pr_{V \sim \Pois(\lambda)}[V = t] \geq \frac{\pr_{V\sim \Pois(\lambda)}[V \geq t]}{e \sqrt{t}}$$
and $t\leq \lambda$,
$$\pr_{V \sim \Pois(\lambda)}[V = t] \geq \frac{\pr_{V\sim \Pois(\lambda)}[V \leq t]}{e \sqrt{t}} .$$
\end{lem}
\begin{proof}
If $t \geq \lambda$ (and $x = t-\lambda$) then by definition and using Stirling's approximation of the factorial we get:
\alequn{\pr_{V \sim \Pois(\lambda)}[V = t] &= e^{-\lambda} \frac{\lambda^t}{t!} \geq  e^{-\lambda} \frac{\lambda^t}{e \sqrt{t} e^{-t} t^t} \\
&= \frac{e^{x}}{e \sqrt t} \lp \frac{\lambda}{\lambda + x}\rp^{\lambda + x} = \frac{e^{x}}{e \sqrt t} \lp \frac{\lambda + x}{\lambda}\rp^{-(\lambda + x)}\\&= \frac{1}{e \sqrt t} e^{-(\lambda+x)\ln\lp 1+\frac{x}{\lambda}\rp-x} \\
& \geq \frac{\pr_{V\sim \Pois(\lambda)}[V \geq \lambda + x]}{e \sqrt{t}},
}
where we used Lemma \ref{lem:poisson-concetrate} to obtain the last inequality.
The case when $t \leq \lambda$ is proved analogously.
\end{proof}

We are now ready to prove Lemma \ref{lem:plurality} which we restate here for convenience.
\begin{lem}
\label{lem:plurality-app}
For $\gamma \geq 0$ let $\rho_\gamma$ denote the categorical distribution $\rho_\gamma$ over $[m]$ such that $\pr_{s \sim \rho_\gamma}[s = m] = \fr{m} + \gamma$ and for all $y \neq m$, $\pr_{s \sim \rho_\gamma}[s = y] = \fr{m} - \frac{\gamma}{m-1}$. For an integer $t$, let $\mnom(t,\rho_\gamma)$ be the multinomial distribution over counts corresponding to $t$ independent draws from $\rho_\gamma$. For a vector of counts $\bc$, let $\argmax(\bc)$ denote the index of the largest value in $\bc$. If several values achieve the maximum then one of the indices is picked randomly.
Then for $\lambda \geq 2 m \ln (4m)$ and $\gamma \leq \frac{1}{8\sqrt{\lambda m}}$,
$$\pr_{t \sim \Pois(\lambda), \bc \sim \mnom(t,\rho_\gamma)} [\argmax(\bc) = m] \geq \fr{m} + \Omega \lp \frac{\sqrt{\lambda}\gamma}{\sqrt{m}}\rp.$$
\end{lem}
\begin{proof}
Let $\bc=(c_1,\ldots,c_m)$ denote the vector of label counts sampled from $\mnom(t,\rho_\gamma)$ for $t$ sampled randomly from $\Pois(\lambda)$.
We first use Fact \ref{fac:poissonize} to conclude that $\bc$ is distributed according to $$\Pois(\lambda') \times \cdots \times \Pois(\lambda')\times \Pois\lp\lambda' + \frac{\gamma\lambda m}{m-1}\rp$$
for $\lambda' = \lp\fr{m} - \frac{\gamma}{m-1}\rp \lambda$.

The next step is to reduce the problem to that of analyzing the product distribution of identical Poisson random variables. Specifically, we view the count of the ``true" label $m$ as the sum of two independent Poisson random variables $c'_m \sim \Pois(\lambda')$ and $d_m \sim \Pois\lp\frac{\gamma\lambda m}{m-1}\rp.$
We also denote by $\bc'$ the vector $(c_1,\ldots,c_{m-1},c'_m)$. Note that $\bc'$ consists of independent and identically distributed samples from $\Pois\lp \lambda' \rp$.

Let $z=\max_{j\in [m-1]} c_j$. By definition, if $c_m > z$ then $\argmax(\bc) = m$ and if $c_m = z$ then $\pr[\argmax(\bc) = m] \leq 1/2$, where the probability is taken solely with respect to the random choice of the index that maximizes the count.
This implies that if $d_m \geq 1$ then
$$\pr[\argmax(\bc) = m] \geq \pr[\argmax(\bc') = m] + \fr{2} \ind(c'_m \in [z-d_m+1,z] ).
$$
Now taking the probability over the random choice of $\bc$ we get
\equ{\pr[\argmax(\bc) = m] \geq \pr[\argmax(\bc') = m]  +  \fr{2} \pr[c'_m \in [z-d_m+1,z]] .
\label{eq:advantage-bound}}

By symmetry of the distribution of $\bc'$ we obtain that
$\pr[\argmax(\bc') = m] = \fr{m}$. To analyze the second term, we first consider the case where $\E[d_m] = \frac{\gamma\lambda m}{m-1}\leq 1$. For this case we simply bound
\equ{\pr[c'_m \in [z-d_m+1,z]] \geq \pr[c'_m = z] \cdot \pr[d_m \geq 1] \label{eq:decompose-small} } (recall that  $d_m$ and $c'_m$ are independent).

By definition of $\Pois\lp\frac{\gamma\lambda m}{m-1}\rp$ we obtain that
 \equ{\pr[d_m \geq 1]  = 1 - e^{-\gamma\lambda m/(m-1)} \geq \frac{\lambda \gamma m}{2(m-1)} \geq \frac{\lambda \gamma}{2} , \label{eq:large_dm}}
 where we used the fact that $e^{-a} \leq 1-a/2$ whenever $a \leq 1$ and our assumption that $\frac{\gamma\lambda m}{m-1} \leq 1$.

Hence it remains to lower bound $\pr[c'_m = z]$. To this end, let $u$ and $v$ be the $1/2$ and $1-1/(4m)$ quantiles of $\Pois(\lambda')$, respectively. That is, $u=\max\{t \cond \pr_{V\sim \Pois(\lambda')}[V \geq t] \geq 1/2\}$ and $v=\max\{t \cond \pr_{V\sim \Pois(\lambda')}[V \geq t] \geq 1-1/(4m)\}$. By the union bound,
$\pr[z \geq v+1] \leq \frac{m-1}{4m} < 1/4$. In addition, by the standard properties of Poisson distribution
$\pr_{V\sim \Pois(\lambda')}[V \geq \lfloor \lambda' \rfloor] \geq 1/2$ which implies that $\pr[z \geq \lfloor \lambda' \rfloor] \geq 1/2$ and thus $u \geq \lfloor \lambda' \rfloor$.

Thus we have an interval such that $$\pr[z  \in [u, v]] \geq \fr{4} .$$ By Lemma \ref{lem:poisson-tail-density}, we have that for every $t \in [u, v]$,
\equ{\pr[c'_m = t] = \pr_{V \sim \Pois(\lambda')}[V = t] \geq \frac{\pr_{V\sim \Pois(\lambda')}[V \geq t]}{e \sqrt{t}} \geq \frac{1}{4e \sqrt{v} m}.\label{eq:upper-density}}
Observe that by our assumption, $\lambda \geq 2\ln(4m)$ and $\frac{\gamma}{m-1} \lambda \leq \lambda'/2$.  Hence $\lambda' \geq \ln(4m)$. By Lemma \ref{lem:poisson-concetrate} this implies that $$v \leq \lambda' + 3\sqrt{\lambda' \ln(4m)} \leq 4\lambda' .$$

Using the independence of $z$ and $c'_m$ we can conclude that
$$ \pr[c'_m = z] \geq \pr[z  \in [u, v]] \cdot \min_{t \in \{u,u+1,\ldots, v\}} \pr[c'_m = t]\geq \fr{4} \cdot \frac{1}{4e \sqrt{v} m} \geq \frac{1}{32 e \sqrt{\lambda'} m}.$$

By combining this bound with eq.\eqref{eq:large_dm}, plugging it into eq.\eqref{eq:decompose-small} and recalling that $\lambda' = \lp\fr{m} - \frac{\gamma}{m-1}\rp \lambda \geq \frac{\lambda}{2m}$ we obtain that
$$ \pr[\argmax(\bc) = m] \geq \fr{m} + \fr{2} \cdot \frac{\lambda \gamma}{2} \cdot \frac{1}{32 e \sqrt{\lambda'} m} = \fr{m} + \Omega\lp \frac{\sqrt{\lambda}\gamma}{\sqrt{m}}\rp .$$

We now consider the other case where $\E[d_m] = \frac{\gamma\lambda m}{m-1} > 1$ which requires a similar but somewhat more involved treatment. We first note that we can assume that $\frac{\gamma\lambda m}{m-1} \geq 12$. For the case when $\frac{\gamma\lambda m}{m-1}\in [1,12]$ we note that it holds that $$\pr[d_m \geq 1] \geq 1-e^{-1} > \fr{20} \gamma\lambda .$$ Thus we can still use the same analysis as before to obtain our claim. By Lemma \ref{lem:poisson-concetrate}, for $\nu \geq 12$, $\pr_{V\sim \Pois(\nu)}[V \in [\nu/2,2\nu]] \geq 1/2$. In particular, under the assumption that $\frac{\gamma\lambda m}{m-1}\geq 12$, $$\pr \lb d_m \in \lb \frac{\gamma\lambda m}{2(m-1)},\frac{2\gamma\lambda m}{m-1} \rb\rb \geq \fr{2} .$$
We also define $u$ and $v$ as before. Using independence of $z$ and $d_m$ we obtain that with probability at least $\fr{4}\cdot\fr{2}$,
we have that $d_m \in \lb \frac{\gamma\lambda m}{2(m-1)},\frac{2\gamma\lambda m}{m-1} \rb$ and $z\in [u,v-1]$. In particular, with probability at least $1/8$, $d_m \geq  \frac{\gamma\lambda m}{2(m-1)}> \gamma\lambda/2$ and $[z-d_m+1,z] \subseteq \lb u',v'\rb$, where $u' =u-\frac{2\gamma\lambda m}{m-1}$ and $v' = v- \frac{\gamma\lambda m}{2(m-1)}$. The interval $[z-d_m+1,z]$ includes $d_m$ integer points and therefore
\equ{\pr[c'_m \in [z-d_m+1,z]] \geq \fr{8} \cdot \frac{\gamma\lambda}{2} \cdot \min_{t \in \{u',u'+1,\ldots, v'\}} \pr[c'_m = t] \label{eq:reduce-to-density}.}
To analyze the lowest value of the probability mass function of $\Pois(\lambda')$ on the integers in the interval $[u',v']$ we first note that
$v' \leq v$ and thus for $t \in [\lambda',v']$ our bound in eq.~\eqref{eq:upper-density} applies. For $t\in [u',\lambda')$ we will first prove that under the assumptions of the lemma $u' \geq \lambda' - \sqrt{\lambda'}$ and then show that for $t \in [\lambda' - \sqrt{\lambda'},\lambda')$, $\pr[c'_m = t] \geq \frac{1}{e^{2}\sqrt{\lambda'}}$. Plugging this lower bound together with the one in eq.~\eqref{eq:upper-density} into eq.~\eqref{eq:reduce-to-density} we obtain the claim:
$$\pr[c'_m \in [z-d_m+1,z]] \geq \frac{\gamma\lambda}{16} \cdot \min\left\{ \frac{1}{8e\sqrt{\lambda'} m}, \frac{1}{e^{2}\sqrt{\lambda'}} \right\} =  \Omega\lp \frac{\sqrt{\lambda}\gamma}{\sqrt{m}}\rp . $$
We now complete these two missing steps. By our definition of $u'$,
$$u' \geq \lfloor \lambda' \rfloor - \frac{2\gamma\lambda m}{m-1} \geq \lambda' - 4\gamma\lambda - 1 \geq \lambda' - \sqrt{\lambda'}$$ which follows from the assumption that $\gamma \leq \frac{1}{8\sqrt{\lambda m}}$ and assumption $\lambda \geq 2 m \ln (4m)$  implying that $\lambda' > 5$.
Now by Lemma \ref{lem:poisson-tail-density} and, using the monotonicity of the pmf of $\Pois(\lambda')$ until $\lambda'$, we have that for $t \in [\lambda' - \sqrt{\lambda'},\lambda')$,
\alequn{\pr[c'_m = t] & \geq  \frac{e^{-(\lambda' - \sqrt{\lambda'}) \ln \lp \frac{\lambda' - \sqrt{\lambda'}}{\lambda'}\rp-\sqrt{\lambda'}}}{e \sqrt{\lambda' - \sqrt{\lambda'}}}\\
& \geq \frac{e^{(\lambda' - \sqrt{\lambda'}) \lp \fr{\sqrt{\lambda'}} - \fr{2\lambda'}\rp-\sqrt{\lambda'}}}{e \sqrt{\lambda' - \sqrt{\lambda'}}} \\
= & \frac{e^{-\fr{2} + \fr{2\sqrt{\lambda'}}}}{e \sqrt{\lambda' - \sqrt{\lambda'}}} \geq \fr{e^{2} \sqrt{\lambda'}} ,
}
where we used the Taylor series of $\ln(1-x) = -x -x^2/2-x^3/3-\ldots $ to obtain the second line.
 \end{proof}

\else

\input{mc-appendix}
\fi

\end{document}